\documentclass[12pt]{article}
\usepackage[top=1.in, bottom=1in, left=1.1in, right=1.1in]{geometry}

\usepackage{amsmath,amsbsy,amsfonts,amssymb,amsthm,dsfont,fullpage,units,bm}
\usepackage{graphicx}
\usepackage{subcaption}
\usepackage{color,cases}
\usepackage{hyperref, todonotes}
\usepackage{authblk}
\usepackage{stmaryrd}
\usepackage{colortbl}  
\usepackage{float}

\hypersetup{
colorlinks = true,
citecolor = {blue},
citebordercolor = {white}
}

\theoremstyle{plain}
\begingroup
\newtheorem{theorem}{Theorem}
\newtheorem*{theorem*}{Theorem}
\newtheorem*{"theorem"}{``Theorem''}

\newtheorem{proposition}[theorem]{Proposition}

\newtheorem{lemma}[theorem]{Lemma}
\endgroup

\theoremstyle{definition}
\begingroup
\newtheorem{definition}[theorem]{Definition}
\endgroup

\theoremstyle{remark}
\begingroup
\newtheorem{remark}[theorem]{Remark}
\newtheorem{example}[theorem]{Example}
\endgroup

\DeclareMathOperator{\rad}{Rad}
\DeclareMathOperator{\argmin}{argmin}
\newcommand{\bi}{\begin{itemize}}
\newcommand{\ei}{\end{itemize}}

\renewcommand{\d}{\mathrm{d}}

\providecommand{\R}{{\ensuremath{\mathbb{R}}}}

\newcommand{\beq}{\begin{equation}}
\newcommand{\eeq}{\end{equation}}
\newcommand{\beqn}{\begin{eqnarray}}
\newcommand{\eeqn}{\end{eqnarray}}
\newcommand{\bsub}{\begin{subequations}}
\newcommand{\esub}{\end{subequations}}
\newcommand{\bpm}{\begin{pmatrix}}
\newcommand{\epm}{\end{pmatrix}}

\newcommand{\CR}{\mathcal{R}}
\newcommand{\cB}{\mathcal{B}}
\newcommand{\cJ}{\mathcal{J}}

\newcommand{\E}{\mathbb{E}}
\newcommand{\B}{\mathcal{B}}

\newcommand\Def{\stackrel{\textrm{def}}{=}}

\newcommand{\cH}{\mathcal{H}}
\newcommand{\cL}{\mathcal{L}}
\newcommand{\cP}{\mathcal{P}}
\newcommand{\cR}{\mathcal{R}}
\newcommand{\cD}{\mathcal{D}}

\newcommand{\cF}{\mathcal{F}}

\newcommand{\W}{\mathcal{W}}
\newcommand{\A}{\mathcal{A}}

\newcommand{\EE}{{\mathbb{E}}}

\newcommand{\RR}{\mathbb{R}}
\newcommand{\N}{\mathbb{N}}

\newcommand{\balpha}{\bm{\alpha}}

\newcommand{\ba}{\bm{a}}
\newcommand{\bb}{\bm{b}}

\newcommand{\bz}{{\bm{z}}}

\newcommand{\wb}{{\bm{w}}}
\newcommand{\xb}{{\bm{x}}}
\newcommand{\bu}{{\bm{u}}}

\newcommand{\bx}{{\bm{x}}}
\newcommand{\by}{{\bm{y}}}

\newcommand{\bv}{{\bm{v}}}
\newcommand{\bw}{\bm{w}}
\newcommand{\bfm}{\bm{m}}
\renewcommand{\W}{\mathcal{W}}

\newcommand{\bB}{{\bm{B}}}
\newcommand{\bI}{{\bm{I}}}
\newcommand{\bW}{{\bm{W}}}

\newcommand{\bU}{{\bm{U}}}
\newcommand{\bV}{{\bm{V}}}
\renewcommand{\P}{\mathbbm{P}}

\newcommand{\supp}{\text{supp }}

\newcommand{\dist}{\text{dist}}
\newcommand{\eps}{\varepsilon}

\newcommand{\bR}{\mathbb{bR}}

\newcommand{\veps}{\varepsilon}

\newcommand{\Risk}{\mathcal{R}}

\newcommand{\lip}{\textit{Lip}_{\{\rho_t\}}}
\newcommand{\Lip}{\textit{Lip}}
\newcommand{\bE}{\mathbb{E}}
\renewcommand{\P}{{P}}

\newcommand\showlabel{\addtocounter{equation}{1}\tag{\theequation}}

\title{Towards a Mathematical Understanding of \\ Neural Network-Based Machine Learning: \\
what we know and what we don't }

\author[1,2]{Weinan E \footnote{Also at Beijing Institute of Big Data Research.}
                   \thanks{\texttt{weinan@math.princeton.edu}}}
\author[3]{Chao Ma \thanks{\texttt{chaoma@stanford.edu}}}
\author[2]{Stephan Wojtowytsch \thanks{\texttt{stephanw@princeton.edu}}}
\author[2]{Lei Wu \thanks{\texttt{leiwu@princeton.edu}}}

\affil[1]{Department of Mathematics, Princeton University}
\affil[2]{Program in Applied and Computational Mathematics, Princeton University}
\affil[3]{Department of Mathematics, Stanford University}

\begin{document}
\maketitle

{
  \hypersetup{linkcolor=black}
  \tableofcontents
}

\section{Introduction}

Neural network-based machine learning is both very powerful and very fragile.
On the one hand, it can be used to approximate functions in very high dimensions with the efficiency and
accuracy never possible before. This has opened up brand new possibilities in a wide spectrum of different disciplines.
On the other hand, it has got the reputation of being somewhat of a ``black magic'':  Its success depends on lots of
tricks, and parameter tuning can be  quite an art.
The main objective for a mathematical study of machine learning is to 
\begin{enumerate}
\item explain the reasons behind the success and the subtleties, and
\item propose new models that are  equally successful but much less fragile.
\end{enumerate}
We are still quite far from completely achieving these goals but it is fair to say that a reasonable big picture is emerging.

The purpose of this article is to review the main achievements towards the first goal  
and discuss the main remaining puzzles. 
In the tradition of good old applied mathematics, we will not only give attention to rigorous mathematical results,
but also discuss the insight we have gained from careful numerical experiments as well as 
the analysis of simplified models. 

At the moment much less attention has been given to the second goal.
One proposal that we should mention is the continuous formulation advocated in \cite{e2019continuous}.
The idea there is to first formulate  ``well-posed'' continuous models of machine learning problems and then discretize
to get concrete algorithms.
What makes this proposal attractive is the following:
\bi
\item many existing machine learning models and algorithms can be recovered in this way in a {\it scaled} form;
\item there is evidence suggesting that indeed  machine learning models obtained this way is more robust with
respect to the choice of hyper-parameters than conventional ones (see for example Figure \ref{heatmap} below);
\item new models and algorithms are borne out naturally in this way.  One particularly interesting example is the
maximum principle-based training algorithm for ResNet-like models \cite{li2017maximum}.
\ei
However, at this stage one cannot yet make the claim that the continuous formulation is the way to go.
For this reason we will postpone a full discussion of this issue to future work.

\subsection{The setup of supervised learning}

The model problem of supervised learning which we focus on in this article can be formulated as follows: given a dataset $ S=\{(\bx_i, y_i=f^*(\bx_i)), i \in [n]\}$,  approximate $f^*$ as accurately as we can.
If $f^*$ takes continuous values, this is called a regression problem. If $f^*$ takes discrete values, this is called a classification problem.

We will focus on the regression problem.
For simplicity, we will neglect the so-called ``measurement noise'' since it does not change much the big picture that we will
describe, even though it does matter for a number of important specific issues.
We will assume $\bx_i \in X:= [0, 1]^d$, and we denote by $P$  the distribution of $\{\bx_i \}$.
We also assume for simplicity that   $\sup_{\bx\in X}|f^*(\bx)|\leq 1$.

 Obviously this is a problem of function approximation.
    As such, it can either be regarded as a problem in numerical analysis or a problem in statistics.
    We will take the former viewpoint since it is more in line with the algorithmic and analysis issues that we will study.
    
The standard procedure for supervised learning is as follows:
\begin{enumerate}
 \item Choose a hypothesis space,  the set of trial functions,  which will be denoted by $\cH_m$. In classical numerical analysis, one often uses  polynomials or piecewise polynomials. In modern machine learning, it is much more popular to use neural network models. { The subscript $m$ characterizes the size of the model. It can be the number of parameters or neurons, and will be specified later for each model.}
  
 \item Choose a loss function.  Our primary goal is to fit the data. Therefore the most popular choice is the
      ``empirical risk'':
       $$\hat{\mathcal{R}}_n(f) = \frac1n \sum_i (f(\bx_i) - y_i)^2 = \frac 1n \sum_i (f(\bx_i) - f^*(\bx_i))^2  $$
    Sometimes one adds some regularization terms.
    
 \item Choose an optimization algorithm and  the hyper-parameters.
    The most popular choices are gradient descent (GD), stochastic gradient descent (SGD) 
   and advanced optimizers such as Adam \cite{kingma2014adam}, RMSprop \cite{Tieleman2012}.
 
\end{enumerate}

The overall objective is to minimize the ``population risk'', also known as  the ``generalization error'':
$$\mathcal{R}(f) = \EE_{\bx \sim P} (f(\bx) - f^*(\bx))^2$$
{In practice, this is estimated on a finite data set (which is unrelated to any data used to train the model) and called test error, whereas the empirical risk (which is used for training purposes) is called the training error.
}

\subsection{The main issues of interest}

From a mathematical perspective, there are three important issues that we need to study:

\begin{enumerate}
    \item 
Properties of the hypothesis space. In particular,  what kind of functions can be approximated efficiently by a
particular machine learning model?  What can we say about the generalization gap, i.e.  the difference between training and test errors.
{
\item Properties of the loss function.  The loss function defines the variational problem used to find the
solution to the machine learning problem.  Questions such as the  landscape of the variational problem are obviously important. The landscape { of neural network models} is typically non-convex, and there may exist many saddle points and bad local minima.
}

\item  Properties of the training algorithm.  Two obvious questions are: Can we optimize the loss function using the selected training algorithm?  
Does the solutions obtained from training generalize well { to  test data}?
\end{enumerate}

The second and third issues are closely related.
In the under-parametrized regime (when the size of the training dataset is larger than the number of free parameters in the hypothesis
space),  this loss function largely determines the solution of the machine learning model.
In the opposite situation, the over-parametrized regime, this is no longer true. Indeed it is often the case that there are infinite number of
global minimizers of the loss function. Which one is picked depends on the details of the training algorithm.

The most important parameters that we should keep in mind are:
\bi
\item $m$: number of free parameters
\item $n$: size of the training dataset
\item $t$: number of training steps
\item $d$: the input dimension.
\ei
Typically we are interested in the situation when  $m, n, t \rightarrow \infty$ and $d \gg 1$.

\subsection{Approximation and estimation errors}

Denote by $\hat{f}$ the output of the machine learning (abbreviated ML)  model.
Let 
$$f_m = \argmin_{f \in {\cH_m}} \mathcal{R}(f) 
$$
We can decompose the error $f^* -\hat{f}$ into:
$$ f^* - \hat{f} = f^* - f_m + f_m - \hat{f} $$
$f^* - f_m$ is the {\it approximation error}, due entirely to the choice of the
hypothesis space.
$f_m - \hat{f}$ is the {\it estimation error}, the additional error due to the fact
that we only have a finite dataset.

To get some basic idea about the approximation error, note that  classically when 
 approximating functions using polynomials, piecewise polynomials, or  truncated Fourier series,
 the error  typically satisfies
$$
 \|f - f_m\|_{L^2(X)} \le C_0 m^{-\alpha/d} \|f\|_{H^{\alpha}(X)}
$$
where $H^{\alpha}$ denotes the Sobolev space of order $\alpha$.
The appearance of $1/d$ in the exponent of $m$
is a signature of an important phenomenon, the {\bf curse of dimensionality (CoD)}:
The number of parameters required to achieve certain accuracy depends exponentially on
the dimension.
For example, if we want $m^{-\alpha/d} = 0.1$, then 
we need $m= 10^{d/\alpha} = 10^d$, if $\alpha =1$.

At this point, it is useful to recall the one problem that
has been extensively studied in high dimension: the problem of
evaluating an integral, or more precisely, computing an expectation.
Let $g$ be a function defined on $X$. We are interested in computing approximately
$$I(g) = \int_{X} g(\bx) d \bx = \EE  g
$$
where the expectation is taken with respect to the uniform distribution.
Typical grid-based quadrature rules, such as the Trapezoidal rule and the Simpson's rule,
all suffer from  CoD. The one algorithm that does not suffer from CoD
is the Monte Carlo algorithm which works as follows.
Let $\{\bx_i\}_{i=1}^n$ be a set of
independent, uniformly distributed random variables on $X$.
Let
$$I_n(g) = \frac 1n  \sum_{i=1}^n g(\bx_i) $$
Then a simple calculation gives us
\begin{equation}
\EE(I(g)-I_n(g))^2 = \frac{\mbox{Var}(g)}n, \quad 
\mbox{Var}(g) = \int_X g^2(\bx) d\bx - \left(\int_X g(\bx) d\bx\right)^2
\label{MC-rate} 
\end{equation}
The $O(1/\sqrt{n})$ rate is independent of $d$.
It turns out that this rate is almost the best one can hope for.

In practice, $\mbox{Var}(g)$ can be very large in high dimension. 
Therefore, variance reduction techniques are crucial in order to make Monte Carlo methods truly practical.

Turning now to the estimation error. 
Our concern is how the approximation produced by the machine learning algorithm
behaves away from the training dataset, or in practical terms, whether the training
and test errors are close.
Shown in Figure \ref{fig: runge} is the classical Runge phenomenon for interpolating 
functions on uniform grids using high order polynomials. One can see that while
on the training set, here the grid points, the error of the interpolant is 0, away from
the training set, the error can be very large. 

\begin{figure}[!h]
    \centering
    \includegraphics[width=0.45\textwidth]{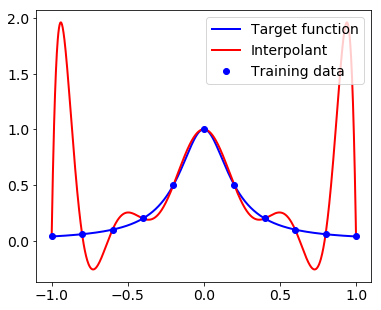}
    \vspace{-2mm}
    \caption{\small The Runge phenomenon with target function $f^*(x) = (1+25\,x^2)^{-1}$.}
    \label{fig: runge}
\end{figure}

It is often easier to study a related quantity, the generalization gap. 
{ Consider the solution that minimizes the empirical risk, $\hat{f} = \argmin_{f \in \cH_m} \hat{\mathcal{R}}_n(f)$. }
The ``generalization gap'' of $\hat{f}$ is the quantity  $|{\mathcal{R}}(\hat{f}) - \hat{\mathcal{R}}_n(\hat{f})|$.
Since it is equal to $|I(g) - I_n(g)|$ with
$g(\bx) = (\hat{f}(\bx) - f^*(\bx))^2$, one 
might be tempted to conclude that
$$ \mbox{generalization} \, \mbox{gap} = O(1/\sqrt{n})
$$
based on \eqref{MC-rate}.
This is NOT necessarily true since  $\hat{f}$ is highly correlated with $\{\bx_i\}$.
In fact, controlling this gap is among the most difficult problems in ML.

Studying the correlations of $\hat{f}$ is a rather impossible problem.
Therefore to estimate the generalization gap, we resort to the { uniform} bound:
\begin{equation}
\label{sup-gap}
|{\mathcal{R}}(\hat{f}) - \hat{\mathcal{R}}_n(\hat{f})|
\le \sup_{f \in \cH_m} |{\mathcal{R}}({f}) - \hat{\mathcal{R}}_n({f})|
= \sup_{f \in \cH_m} |I(g) - I_n(g)|
\end{equation}
The RHS of this equation depends heavily on the nature of $\cH_m$.
If we take $\cH_m$ to be the unit ball in the Lipschitz space, we have \cite{fournier2015rate}
{ $$ \sup_{\|h\|_{Lip} \le 1} |I(g) - I_n(g)| \sim  \frac {1}{n^{1/d}} $$
with $g=(h-f^*)^2$. }
This gives rise to {\bf CoD for the size of the dataset} (commonly referred to as ``sample complexity'').
However if we take $\cH_m$ to be the unit ball in the Barron space, to be
defined later, we have
$$ \sup_{\|h\|_{\mathcal{B}} \le 1} |I(h) - I_n(h)| \sim \frac {1}{\sqrt{n}} $$
This is the kind of estimates that we should look for.

Assuming that all the functions under consideration are bounded,
the problem of estimating the RHS of \eqref{sup-gap} reduces to the estimation
of $\sup_{h \in \cH_m} |I(h) - I_n(h)|$.  One way to do this is to use the notion of
Rademacher complexity \cite{bartlett2002rademacher}.

\begin{definition}
Let $\mathcal{F}$ be a set of functions, and $S=(\bx_1, \bx_2, ..., \bx_n)$ be a set of data points. Then, the Rademacher complexity of $\mathcal{F}$ with respect to $S$ is defined as
\begin{equation*}
\rad_S(\mathcal{F}) = \frac{1}{n}\mathbb{E}_\xi \left[\sup_{h\in\mathcal{F}}\sum\limits_{i=1}^n \xi_ih(\bx_i)\right],
\end{equation*}
where $\{\xi_i\}_{i=1}^n$ are i.i.d. random variables taking values $\pm1$ with equal probability.
\end{definition}

Roughly speaking, Rademacher complexity quantifies the degree to which
functions in the hypothesis space can fit random noise on the given dataset.
It bounds the quantity of interest, $\sup_{h \in \cH} |I(h) - I_n(h)|$,
from above and below.

\begin{theorem}
 For any $\delta \in (0, 1)$, with probability at least $1 - \delta$
    over the random samples $S = (\bx_1, \cdots, \bx_n)$, we have
    \[
      \sup_{h\in\mathcal{F}} \left| \EE_\bx \left[h(\bx)\right] - \frac{1}{n}
      \sum_{i=1}^n h(\bx_{i}) \right|
      \le 2 \rad_S(\mathcal{F}) + \sup_{h\in\mathcal{F}}
      \left\|h\right\|_\infty \sqrt{\frac{\log(2/\delta)}{2n}}.
    \]
    \[
      \sup_{h\in\mathcal{F}} \left| \EE_\bx \left[h(\bx)\right] - \frac{1}{n}
      \sum_{i=1}^n h(\bx_{i}) \right|
      \ge \frac 12  \rad_S(\mathcal{F}) - \sup_{h\in\mathcal{F}}
      \left\|h\right\|_\infty \sqrt{\frac{\log(2/\delta)}{2n}}.
    \]
\end{theorem}   

For a proof, see for example \cite[Theorem 26.5]{shalev2014understanding}. {For this reason, a very important part of theoretical machine learning is to study the Rademacher complexity of a given hypothesis space.}

It should be noted that there are other ways of analyzing the generalization gap,
such as the stability method \cite{bousquet2002stability}.

{
\paragraph{Notations.} For any function $f: \RR^m\mapsto\RR^n$, let  $\nabla f = (\frac{\partial f_i}{\partial x_j})_{i,j}\in\RR^{n\times m}$ and $\nabla^T f = (\nabla f)^T$. We use $X\lesssim Y$ to mean that $X\leq CY$ for some absolute constant $C$. For any $\bx\in\RR^d$, let $\tilde{\bx}=(\bx^T,1)^T\in\RR^{d+1}$. 
Let $\Omega$ be a subset of $\RR^d$, and denote by $\cP(\Omega)$ the space of probability measures. Define $\cP_2(\Omega)=\{\,\mu\in \cP(\Omega): \int \|\bx\|_2^2 d\mu(\bx)<\infty\}$.
We will also follow the convention in probability theory to use
$\rho_t$ to denote  the value of $\rho$ at time $t$.
}

\section{Preliminary remarks}

In this section we  set the stage for our discussion by going over some of the classical results as well as recent 
qualitative  studies that are of general interest.

\subsection{Universal approximation theorem and the CoD}

Let us consider the two-layer neural network hypothesis space:
$$ \cH_m = \{ f_m(\bx) = { \sum_{j=1}^m}  a_j \sigma(\wb_j^T \xb) \}
$$
where $\sigma $ is a nonlinear function, the activation function.
The Universal Approximation Theorem (UAT) states that under very mild conditions,  any continuous
function  can be uniformly approximated on compact domains
by neural network functions.

\begin{theorem}\cite[Theorem 5]{cybenko1989approximation}
If $\sigma$ is sigmoidal, in the sense that
$\lim_{z \rightarrow -\infty} \sigma(z) = 0,
\lim_{z \rightarrow \infty} \sigma(z) = 1$, then any function in
$C([0, 1]^d)$ can be approximated uniformly
by two layer neural network functions.
\end{theorem}

This result can be extended to any activation functions that are not exactly a polynomial \cite{leshno1993multilayer}.

UAT plays the role of Weierstrass Theorem on the approximation of continuous functions by polynomials.
It is obviously an important result,  but it is of limited use  due to the
lack of quantitative information about the error in the approximation.
For one thing, the same conclusion can be drawn for polynomial approximation, which we know is of limited use
in high dimension.

Many quantitative error estimates have been established since then.  But most of these estimates suffer from CoD.
The one result that stands out is the estimate proved by Barron \cite{barron1993universal}:
\begin{equation}
\inf_{f_m \in \cH_m} \|f^* - f_m \|^2_{L^2(P)} \lesssim \frac{\Delta(f^*)^2}m
\label{Barron-1993}
\end{equation}
where $\Delta(f) $ is a norm defined by
\[
\label{spectral-norm}
\Delta(f):=\inf_{\hat{f}}\int_{\RR^d} \|\omega\|_1\, |\hat{f}(\omega) | d \omega < \infty,
\]
where $\hat{f}$ is the Fourier transform of an  extension of $f$ to $L^2(\RR^d)$.
The convergence rate in \eqref{Barron-1993} is independent of dimension.
{ However, the  constant $\Delta(f^*)$  could be dimension-dependent since it makes use of the Fourier transform.}
As the dimensionality goes up, the number of derivatives
required to make $\Delta$ finite also goes up. This is distinct from the CoD, which is the dimension-dependence of the rate.

\subsection{The loss landscape of large neural network models}

The landscape of the loss function for large neural networks was studied in \cite{choromanska2015loss} using an analogy
with high dimensional spherical spin glasses and numerical experiments.
The landscape of high dimensional spherical spin glass models has been analyzed in \cite{auffinger2013random}.
It was shown that the lowest critical values of the Hamiltonians of these
models form a layered structure, ordered by the indices of the critical points. They are located in a
well-defined band lower-bounded by the global minimum. The probability of finding critical points
outside the band diminishes exponentially with the dimension of the spin-glass model. 
Choromanska et al \cite{choromanska2015loss} used  the observation that neural networks with independent weights can be mapped onto 
spherical spin glass models, and suggested that the same picture should hold qualitatively for large
neural network models.
They provided numerical evidence to support this suggestion.

This work is among the earliest  that suggests even though it is highly non-convex, 
the landscape of larger neural networks might be simpler  than the ones for smaller networks.

\subsection{Over-parametrization, interpolation and implicit regularization}

Modern deep learning often works in the over-parametrized regime where the number of free parameters is larger
than the size of the training dataset. This is a  new territory as far as machine learning theory is concerned.
Conventional wisdom would suggest that one should expect overfitting in this regime, i.e.  an increase in the generalization gap.
Whether overfitting really happens is a problem of great interest in theoretical machine learning.
As we will see later, one can show that, overfitting does happen in the highly over-parametrized regime.

An enlightening numerical study of the optimization and generalization properties in this regime was carried out
in \cite{zhang2016understanding}.  Among other things, it was discovered that in this regime, the neural networks are so 
expressive that they can fit any data, no matter how much noise one adds to the data.
Later it was shown by Cooper that the set of global minima with no training error forms a manifold of dimension
$m-n$ \cite{cooper2018loss}.

Some of these global minima generalize very poorly. Therefore an important question is how to select the ones that
do generalize well.
It was suggested in \cite{zhang2016understanding} that by tuning the hyper-parameters of the optimization algorithm, one can obtain
models with good generalization properties without adding explicit regularization. This means that
the training dynamics itself has some implicit regularization mechanism which ensures that bad global minima are not 
selected during training.
Understanding the possible mechanism for this implicit regularization is one of the key issues in understanding modern
machine learning.

\subsection{The selection of topics}

There are several rather distinct ways for a mathematical analysis of machine learning and particularly neural network models:
\begin{enumerate}
\item The numerical analysis perspective. Basically machine learning problems are viewed as (continuous) function approximation  and 
optimization problems, typically in high dimension.  
\item The harmonic analysis perspective. Deep learning is studied from the viewpoint of building hierarchical wavelets-like
 transforms. Examples of such an approach can be found in \cite{Mallat, Taicheng}.
 \item The statistical physics perspective. A particularly important tool is the replica trick. 
 For special kinds of problems, this approach allows us
 to perform asymptotically exact (hence sharp) calculations.  { It is also useful to} study the performance of models and algorithms from the
 viewpoint of phase transitions.  See for example \cite{Lenka, barbier2019optimal, goldt2020gaussian}.
  \item The information theory perspective. See \cite{Abbe} for an example of the kinds of results obtained along this line.
 \item The PAC learning theory perspective. This  is closely related to the information theory perspective. It studies
 machine learning from the viewpoint of complexity theory.  See for example  \cite{Livni}.
\end{enumerate}
{ It is not yet clear} how the different approaches are connected, even though many results may fall in several categories.
  In this review, we will only cover the results obtained along the lines of numerical analysis.  We encourage the reader to consult
  the papers referenced above to get a taste of these alternative perspectives.

Within the numerical analysis perspective,  supervised machine learning and neural network models are still vast topics. 
By necessity, this article focusses on a few aspects which we believe to be key problems in machine learning. { As a model problem, we focus on $L^2$-regression here, where the data is assumed to be of the form $(\bx_i, f^*(\bx_i))$ without uncertainty in the $y$-direction. }

In Section \ref{section spaces}, we focus on the function spaces developed for neural network models. Section \ref{section landscape} discusses the (very short) list of results available for the energy landscape of loss functionals in machine learning. Training dynamics for network weights are discussed in Section \ref{section training} with a focus on gradient descent. The specific topics are selected for two criteria:

\begin{itemize}
\item We believe that they have substantial importance for the mathematical understanding of machine learning.
\item We are reasonably confident that the mathematical models developed so far will over time find their way into the standard language of the theoretical machine learning community.
\end{itemize}

There are many glaring omissions in this article, among them:

\begin{enumerate}
\item Classification problems. While most benchmark problems for neural networks fall into the framework of classification, we focus on the more well-studied area of regression.

\item {Some other common neural network architectures, such as convolutional neural networks, long short-term memory (LSTM) networks, encoder-decoder networks.}

\item The impact of stochasticity. Many training algorithms and initialization schemes for neural networks use random variables. While toy models with standard Gaussian noise are well understood, the realistic case remains out of reach \cite{hu2019mean}.

\item { Simplified neural networks such as linear and quadratic networks.} While these models are not relevant for applications, the simpler model allows for simpler analysis. Some insight is available for these models which has not been achieved in the non-linear case \cite{saxe2013exact,kawaguchi2016deep, arora2018convergence}.

{ 
\item Important ``tricks'', such as dropout, batch normalization, layer normalization, gradient clipping, etc. To our knowledge, these remain empirically useful, but mysterious from a mathematical perspective.
}

\item Highly data-dependent results. The geometry of data is a large field which we avoid in this context. While many data-distributions seem to be concentrated close to relatively low-dimensional manifolds in a much higher-dimensional ambient space, these `data-manifolds' are in many cases high-dimensional enough that CoD, a central theme in this review,  is still an important issue.
\end{enumerate}

\section{The approximation property and the Rademacher complexity of the hypothesis space}\label{section spaces}

The most important issue in classical approximation theory is to identify the function space naturally associated with a 
particular approximation scheme, e.g. approximation by piecewise polynomials on regular grids.
These spaces are typically some Sobolev or Besov spaces, or their variants.
They are the natural spaces for the particular approximation scheme, since one can prove matching direct
and inverse approximation theorems, namely any function in the space can be approximated using the given
approximation scheme with the specified rate of convergence, and conversely any function that can be approximated to the
specified order of accuracy belongs to that function space.

Machine learning is just another way to approximate functions, therefore we can ask similar questions, except that our
main interest in this case is in high dimension.
Any machine learning model hits the `curse of dimensionality' when approximating the class of Lipschitz functions. 
Nevertheless, many important problems seem to admit accurate approximations by neural networks. 
Therefore it is important to understand the class of functions that can be well approximated by a particular machine learning model.

There is one important difference from the classical setting. In high dimension, the rate of 
convergence is limited to the Monte Carlo rate and its variants. There is limited room regarding order of convergence and 
consequently there is no such thing as the order of the space as
is the case for Sobolev spaces.

Ideally,  we would like to accomplish the following:

\begin{enumerate}
    \item Given a type of hypothesis space $\mathcal{H}_m$, say two-layer neural networks,  identify the natural function space associated with them
(in particular, identify a norm  $\|f^*\|_*$) that satisfies:

\bi

\item Direct  approximation theorem:
$$ \inf_{f \in \mathcal{H}_m}  \mathcal{R}(f)  = 
\inf_{f \in \mathcal{H}_m}  \|f - f^*\|^2_{L^2(P)}
\lesssim \frac{\|f^*\|_*^2}{m} $$
\item Inverse approximation theorem: If a function $f^*$ can be approximated efficiently by the functions in 
$\mathcal{H}_m$, as $m \rightarrow \infty$  with some uniform bounds, then $\|f^*\|_*$ is finite.

\ei

\item  Study the generalization gap for this function space.  
One way of doing this is to study the Rademacher complexity of the set
 $\cF_Q = \{ f\, : \, \|f\|_* \le Q  \}$.
Ideally, we would like to have:
\[
    \rad_S(\cF_Q) \lesssim  \frac{Q}{\sqrt{n}} 
\]

\end{enumerate}
If both holds, then a combination gives us, up to logarithmic terms:
\[
\mathcal{R}(\hat{f}) \lesssim \frac{\|f^*\|_*^2}{m} + \frac{\|f^*\|_*}{\sqrt{n}}
\label{estimate-1}
\]

\begin{remark}
It should be noted that what we are really interested in is the quantitative measures of the target function
that control the approximation and estimation errors.  We call these quantities ``norms'' but we are not going to insist that they are really norms.
In addition, we would like to use one norm to control both the approximation and estimation errors.
This way we have one function space that meets both requirements.
However, it could very well be the case that we need different quantities to control different errors.
See the discussion about residual networks below.
This means that we will be content with a generalized version of \eqref{estimate-1}:
\[
\mathcal{R}(\hat{f}) \lesssim \frac{\Gamma(f^*)}{m} + \frac{\gamma(f^*)}{\sqrt{n}}
\label{estimate-2}
\]
We will see that this can indeed be done for the most popular neural network models.
\end{remark}

\subsection{Random feature model}\label{sec: rf}

Let  $\phi(\cdot; \bw) $ be the feature function parametrized by $\bw$, e.g. $\phi(\bx;\bw) = \sigma (\bw^T \bx)$. A random feature model is given by 
\begin{equation}
    f_m(\bx;\ba) = \frac 1m  \sum_{j=1}^m a_j \phi(\bx; \bw_j^0).
\end{equation}
where $\{\bw_j^0\}_{j=1}^m$ are i.i.d random variables drawn from a prefixed distribution $\pi_0$. The collection $\{\phi(\cdot;\bw_j^0)\}$ 
are the random features, $\ba=(a_1,\dots, a_m)^T\in\RR^m$ are the coefficients.  For this model, the natural function space is the reproducing kernel Hilbert space (RKHS) \cite{aronszajn1950theory} induced by the kernel 
\begin{equation}
k(\bx,\bx')=\EE_{\bw\sim\pi_0}[\phi(\bx;\bw)\phi(\bx';\bw)]
\end{equation}
Denote by $\cH_k$ this RKHS. Then for any $f\in \cH_k$, there exists $a(\cdot)\in L^2(\pi_0)$ such that 
\begin{equation}
    f(\bx) = \int a(\bw) \phi(\bx;\bw) d\pi_0(\bw),
\end{equation}
and 
\begin{equation}
    \|f\|_{\cH_k}^2 = \inf_{a \in S_f}\int a^2(\bw) d\pi_0(\bw),
\end{equation}
where $S_f:=\{a(\cdot): f(\bx)=\int a(\bw)\phi(\bx;\bw)d\pi_0(\bw)\}$.
We also define $\|f\|_{\infty}=\inf_{a\in S_f}\|a(\cdot)\|_{L^{\infty}(\pi_0)}$.
For simplicity, we assume  that $\Omega:=\text{supp}(\pi_0)$ is compact. 
Denote $\bW^0=(\bw_1^0,\dots, \bw^0_m)^T\in \RR^{m\times d}$ and $a(\bW^0)=(a(\bw_1^0),\dots, a(\bw^0_m))^T\in \RR^m$.

\begin{theorem}[Direct Approximation Theorem]
 Assume $f^*\in \cH_k$, then there exists  $a(\cdot)$, such that
\[
    \EE_{\bW^0}[\|f_m(\cdot;a(\bW^0))-f^*)\|_{L^2}^2]\leq \frac{\|f^*\|_{\cH_k}^2}{m}.
\]
\end{theorem}

\begin{theorem}[Inverse Approximation Theorem]\label{thm: rf-invers-approx}
Let $(\bw_j^0)_{j=0}^\infty$ be a sequence of i.i.d. random variables drawn from $\pi_0$. Let $f^*$ be a continuous function on $X$. Assume that there exist constants $C$ and a sequence $(a_j)_{j=0}^\infty$ satisfying $ \sup_j |a_j|\leq C$, such that 
\begin{align}
\lim_{m\to\infty}\frac{1}{m}\sum_{j=1}^m a_j \phi(\bx;\bw_j^0) &= f^*(\bx),
\end{align}
 for all $\bx\in X$.  
 Then with probability $1$, there exists a function $a^*(\cdot): \Omega\mapsto\RR$  such that 
\[
f^*(\bx) = \int_{\Omega} a^*(\bw) \phi(\bx;\bw) d\pi_0(\bw),
\]
Moreover, $\|f\|_{\infty}\leq C$.
\end{theorem}

To see how these approximation theory results can play out in a realistic ML setting, 
consider the regularized model:
     \begin{equation*}
\cL_{n,\lambda}(\ba) = \hat{\mathcal{R}}_n(\ba) +
 \frac{1}{\sqrt{n}}\frac{\|\ba\|}{\sqrt{m}},
\end{equation*}
and define the regularized estimator:
 \[
 \hat{\ba}_{n,\lambda} = \argmin \cL_{n,\lambda}(\ba).
 \]

\begin{theorem}\label{thm: rf-a-priori} For any $\delta \in (0,1)$, with probability $1 - \delta$,
the population risk of the regularized estimator satisfies 
\begin{align}
\cR(\hat{\ba}_{n})&\leq \frac{1}{m}\left(\log(n/\delta)\|f^*\|_{\cH_k}^2 + \frac{\log^2(n/\delta)}{m}\|f^*\|_{\infty}^2\right) \\ &\quad +\frac{1}{\sqrt{n}} \left(\|f^*\|_{\cH_k} + \left(\frac{\log(1/\delta)}{m}\right)^{1/4}\|f^*\|_{\infty} + \sqrt{\log(2/\delta)}\right).
\end{align}
\end{theorem}

These results should be standard. However, they do not seem to be available in the literature.  In the
appendix, we provide a proof for these results.

It is worth noting  that the dependence on $\|f\|_{\infty}$ and $\log(n/\delta)$ can be removed by a more sophisticated analysis \cite{bach2017equivalence}. However, to achieve  the rate of $O(1/m+1/\sqrt{n})$, one must make an explicit assumption on the decay rate of eigenvalues of the corresponding kernel operator \cite{bach2017equivalence}.

\subsection{Two-layer neural network model}\label{section barron approximation}

The hypothesis space for two-layer neural networks is defined by:
$$ \cF_m = \{ f_m(\bx) =
\frac 1 m \sum_{j=1}^m  a_j \sigma(\wb_j^T \xb) \}
$$
We will focus on the case when the activation function $\sigma$ is ReLU: $\sigma(z) = \max(z, 0)$.
Many of the results discussed below can be extended to more general activation functions \cite{MaLiWu}.

The function space for this model is called the Barron space (\cite{e2018priori, e2019barron},  see also 
\cite{barron1993universal, klusowski2016risk, wojtowytsch2020representation} and particularly \cite{bach2017breaking}).
Consider the function $f: X = [0,1]^d\mapsto \RR$ of the following form
$$ f(\bx) = \int_{\Omega} a \sigma (\bw^T  \bx ) \rho (da, d\bw)
=\EE_\rho[a\sigma(\bw^T\bx)], \quad \bx \in X$$
where $\Omega = \RR^1 \times \RR^{d+1} $,
$\rho$ is a probability distribution on $\Omega$.
The Barron norm is defined by
$$\|f \|_{\mathcal{B}_p} = \inf_{\rho\in P_f} \left(\EE_\rho [a^p\|\bw\|^p_1] \right)^{1/p} $$
where $P_f:=\{\rho: f(\bx)=\EE_\rho[a\sigma(\bw^T\bx)]\}$. Define
$$
\mathcal{B}_p =  \{ f \in C^0 : \|f\|_{\mathcal{B}_p} < \infty \}
$$
Functions in $\mathcal{B}_p$ are called Barron functions. As shown in \cite{e2019barron}, for the ReLU activation function, 
we actually have $\|\cdot\|_{\mathcal{B}_p} = \|\cdot\|_{\mathcal{B}_q}$ for any $1 \leq p\leq q \leq \infty$. Hence, we will use $\|\cdot\|_{\cB}$ and $\cB$ denote the Barron norm and Barron space, respectively. 

\begin{remark}
Barron space and Barron functions are named in reference to the article \cite{barron1993universal} which was the first to recognize and rigorously establish the advantages of non-linear approximation over linear approximation by considering neural networks with a single hidden layer.

It should be stressed that the Barron norm introduced above is not the same as the one used in \cite{barron1993universal},
which was based on the Fourier transform (see \eqref{spectralnorm}).
To highlight this distinction, we will call the kind of norm in \eqref{spectralnorm} {\it spectral norm}.

An important property of the ReLU activation function is the homogeneity property $\sigma(\lambda z) = \lambda\,\sigma(z)$ for all $\lambda>0$. 
 A discussion of the representation for Barron functions with partial attention to homogeneity can be found in \cite{wojtowytsch2020representation}. Barron spaces for other activation functions are discussed in \cite{approximationarticle,MaLiWu}. 
\end{remark}

One natural question is what kind of functions are Barron functions.  The following result gives a partial answer.
\begin{theorem}[Barron and Klusowski (2016)]\label{thm: barron-klusowski2016}
 If 
 \begin{equation}\label{spectralnorm}
 \Delta(f):=\int_{\RR^d} \|\omega\|_1^2 |\hat{f}(\omega) | d \omega < \infty,
 \end{equation}
where $\hat{f}$ is the Fourier transform of $f$, then $f$ can be represented as 
$$f(\bx)= \int_\Omega a \sigma (\bb^T \bx + c) \rho(da, d\bb, dc), \quad \forall \bx \in X$$
where $\sigma (x) = \max(0, x)$. Moreover $\|f\|_{\cB}\leq 2 \Delta(f) + 2 \|\nabla f(0)\|_1 + 2 f(0)$.
\end{theorem}

A consequence of this is that every function in $H^s(\R^d)$ is Barron for $s> \frac d2 + 1$.
In addition, it is obvious that  every finite sum of neuron activations is also a Barron function.
Another interesting example of a Barron function is
 the function $f(\bx) = \|x\|_{\ell^2} =\EE_{\bb\sim\mathcal{N}(0, I_d)}[\sigma(\bb^T\bx)]$.

On the other hand, every Barron function is Lipschitz-continuous. An important criterion to establish that certain functions are {\em not} in Barron space is the following structure theorem.

\begin{theorem}\cite[Theorem 5.4]{wojtowytsch2020representation}\label{theorem structure barron}
Let $f$ be a Barron function. Then $f = \sum_{i=1}^\infty f_i$ where $f_i\in C^1(\R^d\setminus V_i)$ where $V_i$ is a $k$-dimensional affine subspace of $\R^d$ for some $k$ with $0\leq k\leq d-1$.
\end{theorem}

As a consequence, distance functions to curved surfaces are not Barron functions.

The claim that the Barron space is the natural space associated with two-layer networks is justified by the following series of results.

\begin{theorem}[Direct Approximation Theorem, $L^2$-version]
For any $f\in \cB$ and $m\in \mathbb{N}$, there exists a two-layer neural network $f_m$ with $m$ neurons $\{(a_j, \bw_j)\}_{j=1}^m$ such that 
$$
\| f- f_m \|_{L^2(P)} \lesssim  \frac{ \|f \|_{\mathcal{B}}}{\sqrt{m}}.
$$
\end{theorem}

\begin{theorem}[Direct Approximation Theorem, $L^\infty$-version]\label{theorem direct linfty}
For any $f\in \cB$ and $m\in \mathbb{N}$, there exists a two-layer neural network $f_m$ with $m$ neurons $\{(a_j, \bw_j)\}_{j=1}^m$ such that 
$$
\| f- f_m \|_{L^\infty([0,1]^d)} \lesssim  4\,\|f\|_\B{\sqrt \frac{d+1}m}.
$$
\end{theorem}

We present a brief self-contained proof of the $L^\infty$-direct approximation theorem in the appendix. We believe the idea to be standard, but have been unable to locate a good reference for it.

\begin{remark}
In fact, there exists a constant $C>0$ such that
\[
\| f- f_m \|_{L^\infty([0,1]^d)} \leq C\,\|f\|_{\mathcal{B}}\,\frac{\sqrt{\log(m)}}{m^{1/2+{ 1/2d}}}
\]
and for every $\eps>0$ there exists $f\in \mathcal{B}$ such that 
\[
\|f-f_m\|_{L^\infty([0,1]^d)}\geq c\,m^{-1/2 - 1/d {-}\eps},
\]
see \cite{barron1992neural, makovoz1998uniform}. Further approximation results, including in classical functions spaces, can be found in \cite{pinkus1999approximation}.
\end{remark}

\begin{theorem}[Inverse Approximation Theorem]
Let
\[
\mathcal{N}_{C} \Def \{\,\frac{1}{m}\sum_{j=1}^m a_j\sigma(\bw_j^T\bx) \,:\, 
\frac{1}{m}\sum_{j=1}^m |a_j|\|\bw_j\|_1 \leq C, m\in \mathbb{N}^{+}\,\}.
\]
Let $f^*$ be a continuous function. Assume
there exists a constant $C$ and a sequence of functions $f_m \in \mathcal{N}_{C} $
such that
$$ f_m(\bx) \rightarrow f^*(\bx) $$
for all $\bx \in X$, then
there exists a probability distribution $\rho^*$ on $\Omega$, such that
\[
   f^*(\bx) = \int a\sigma(\bw^T\bx) \rho^*(da,d\bw),
\]
for all $\bx \in X$ and $\|f^*\|_{\mathcal{B}} \le C$.
\end{theorem}

Both theorems are proved in \cite{e2018priori}. In addition, it just so happens that the Rademacher complexity is also controlled by a Monte Carlo like rate:

\begin{theorem}[\cite{bach2017breaking}]
Let $\cF_Q = \{ f \in \mathcal{B}, \|f\|_{\mathcal{B}} \le Q  \}$.
Then we have
\[
    \rad_S(\cF_Q) \leq 2Q \sqrt{\frac{2\ln(2d)}{n}}
\]
\end{theorem}

In the same way as before,  one can now consider the regularized model:
\begin{equation*}
\cL_n(\theta) = \hat{\mathcal{R}}_n(\theta) +
\lambda \sqrt{\frac{\log(2d)}{n}} \|\theta\|_{\mathcal{P}}, \quad \quad
        \hat{\theta}_n = \argmin \cL_n(\theta)
    \end{equation*}
where the path norm is defined by:
    \[
        \|\theta\|_{\mathcal{P}} = \frac 1m \sum_{j=1}^m |a_j|\|\bw_j\|_1
    \]

\begin{theorem}\cite{e2018priori}:
 Assume  $f^*: X \mapsto [0,1] \in \mathcal{B}$.
There exist constants  absolute $C_0$, such that for any $\delta > 0$,
        if $\lambda\geq C_0$, then with probability  at least $1-\delta$ over the choice of the training set, we have
        \[
            \mathcal{R}(\hat{\theta}_n) \lesssim \frac{ \|f^*\|_{\mathcal{B}}^2}{m} + \lambda \|f^*\|_{\mathcal{B}} \sqrt{\frac{\log (2d)}{n}} + \sqrt{\frac{\log(n/\delta)}{n}}.
            \label{apriori-2layer}
        \]
\end{theorem}

\subsection{Residual networks } 

Consider a residual network model
\begin{equation}
\begin{aligned}
 \bz_{0,L}(\bx) &= \bV\bx, \nonumber\\
 \bz_{l+1,L}(\bx) &= z_{l,L}(\bx)+\frac{1}{L}\bU_l\sigma \circ(\bW_l z_{l,L}(\bx)), \quad l=0, 1, \cdots, L-1 \\
 f(\bx, \theta) &=  \balpha^T \bz_{L,L}(\bx), \label{eq:resnet}
\end{aligned}
\end{equation}
where $\bx\in\bR^d$ is the input,  $\bV \in \RR^{D\times d}, \bW_l\in\bR^{m\times D}$, $\bU_l\in\bR^{D\times m}, \balpha\in\bR^D$.
 Without loss of generality, we will fix $\bV$ to be
\begin{equation}\label{eq:bV}
    \bV=\left[\begin{array}{l}
         \bI_{d\times d} \\
         \mathbf{0}_{(D-d)\times d}
    \end{array}\right].
\end{equation}
 We use $\Theta:=\{ \bU_1,\dots,\bU_L, \bW_l,\dots,\bW_L, \balpha\}$ to denote all the parameters to be learned from data. 
 
Consider the following ODE system \cite{e2019barron}:
\begin{equation}
\begin{aligned}
\bz(\bx,0) &= \bV\bx, \nonumber \\
\dot{\bz}(\bx,t) &= \bE_{(\bU,\bW)\sim\rho_t} \bU\sigma(\bW\bz(\bx,t)), \\
f_{\balpha,\{\rho_t\}}(\bx) &= \balpha^T \bz(\bx,1).\label{eq:def_zx}
\end{aligned}
\end{equation}
This ODE system can be viewed as the limit of  the residual network \eqref{eq:resnet}  (\cite{e2019barron}).
Consider the following linear ODEs  ($p \ge 1$)
\begin{equation}
\begin{aligned}
  N_p(0) &= \mathbf{1}, \nonumber \\
  \dot{N}_p(t) &=  \left( \bE_{\rho_t}(|\bU||\bW|)^p \right)^{1/p}N_p(t), \label{eq:n_ode}
\end{aligned}
\end{equation}
where $\bm{1}=(1,\dots,1)^T\in\RR^{d}$, $|\mathbf{A}|$ and $\mathbf{A}^q$ are element-wise operations for the matrix $\mathbf{A}$ and positive number $q$. 
\begin{definition}[\cite{e2019barron}] \label{def:dp}
Let $f$ be a function that admits the form $f=f_{\balpha,\{\rho_t\}}$ for a pair of ($\balpha, \{\rho_t\}$), then we define
\begin{equation}
  \|f\|_{\cD_p(\balpha,\{\rho_t\})} = |\balpha|^T N_p(1),
\end{equation}
to be the $\cD_p$ norm of $f$ with respect to the pair ($\balpha$, $\{\rho_t\}$).
Here $|\balpha|$ is obtained from $\balpha$ by taking element-wise absolute values. We define
\begin{equation}
  { \|f\|_{\cD_p} = \inf_{f=f_{\balpha,\{\rho_t\}}} |\balpha|^T N_p(1).} \label{eq:comp_norm}
\end{equation}
to be the $\cD_p$ norm of $f$, and let $\cD_p=\{f: \|f\|_{\cD_p}<\infty\}$ be the flow-induced function space. 
\end{definition}

{
\begin{remark}
In our definition, $\rho_t$ is a distribution on $\bR^{D\times m}\times\bR^{m\times D}$. However, the ODE system~\eqref{eq:def_zx} with large $m$ does not express more functions than just taking $m=1$. Nevertheless, choosing $m\neq1$ may influence the associated function norm as well as the constant on the approximation bound in Theorem~\ref{thm:direct_comp}. We do not explore the effect of different $m$ in this paper. 
\end{remark}
}

Beside $\cD_p$, \cite{e2019barron} introduced another class of function spaces $\tilde{\cD}_p$ that contain functions for which the quantity $N_p(1)$ { and the continuity of $\rho_t$ with respect to $t$} are controlled. 
{ We first provide the following definition of ``Lipschitz condition'' of $\rho_t$.
\begin{definition}\label{def:lip}
Given a family of probability distribution $\{\rho_t,\ t\in[0,1]\}$, the ``Lipschitz coefficient'' of $\{\rho_t\}$,  denoted by $\lip$, is defined as the infimum of all the numbers $L$ that satisfies
\begin{equation}\label{eq:lip}
\left| \bE_{\rho_t}\bU\sigma(\bW\bz)-\bE_{\rho_s}\bU\sigma(\bW\bz) \right|\leq L |t-s||\bz|,
\end{equation}
and 
\begin{equation}\label{eq:lip2}
\left| \left\|\bE_{\rho_t}|\bU||\bW|\right\|_{1,1}-\left\|\bE_{\rho_s}|\bU||\bW|\right\|_{1,1} \right|\leq L |t-s|,
\end{equation}
for any $t,s\in[0,1]$, where $\|\cdot\|_{1,1}$ is the sum of the absolute values of all the entries in a matrix.
The ``Lipschitz norm'' of $\{\rho_t\}$ is defined as
\begin{equation}
\|\{\rho_t\}\|_\Lip=\left\|\EE_{\rho_0}|\bU||\bW|\right\|_{1,1}+\lip.
\end{equation}
\end{definition}
}
\begin{definition}[\cite{e2019barron}]\label{def:dp2}
Let $f$ be a function that satisfies $f=f_{\balpha,\{\rho_t\}}$ for a pair of ($\balpha, \{\rho_t\}$), then we define
\begin{equation}
  \|f\|_{\tilde{\cD}_p(\balpha,\{\rho_t\})} = |\balpha|^T N_p(1)+\|N_p(1)\|_1-D + \|\{\rho_t\}\|_\Lip,
\end{equation}
to be the $\tilde{\cD}_p$ norm of $f$ with respect to the pair ($\balpha$, $\{\rho_t\}$). We define
\begin{equation}
  \|f\|_{\tilde{\cD}_p} = \inf_{f=f_{\balpha,\{\rho_t\}}} \|f\|_{\tilde{\cD}_p(\balpha,\{\rho_t\})}.
\end{equation}
to be the $\tilde{\cD}_p$ norm of $f$. The space $\tilde{\cD}_p$ is defined as all the functions that admit the representation $f_{\balpha,\{\rho_t\}}$ with finite $\tilde{\cD}_p$ norm.
\end{definition}

The space $\cD_p$ and $\tilde{\cD}_p$ are called  {\it flow-induced spaces}.

One can easily see that the ``norms'' defined here are all non-negative quantities (despite the $-D$ term), even though it is not
clear that they are really norms.
The following embedding theorem shows that flow-induced function space is larger than Barron space.
\begin{theorem}\label{thm:barron}
    For any function $f\in\cB$, and $D\geq d+2$ and $m\geq1$, we have $f\in\tilde{\cD}_1$, and
    \begin{equation}
        \|f\|_{\tilde{\cD}_1}\leq 2\|f\|_\cB + 1.
    \end{equation}
\end{theorem}

Finally, we define a discrete ``path norm'' for residual networks.
\begin{definition}\label{eq:l1_path}
For a residual network defined by~\eqref{eq:resnet} with parameters $\Theta=\{\balpha, \bU_l, \bW_l, l=0,1,\cdots,L-1\}$, we define the $l_1$ path norm of $\Theta$ to be
\begin{equation}\label{eqn: l1_path}
\|\Theta\|_{\cP}=|\balpha|^T\prod_{l=1}^L\left(I+\frac{1}{L}|\bU_l||\bW_l|\right)\bm{1}.
\end{equation}
\end{definition}

With the definitions above, we are ready to state the direct and inverse approximation theorems for the flow-induced function spaces~\cite{e2019barron}.

{
\begin{theorem}[Direct Approximation Theorem]\label{thm:direct_comp}
Let $f\in\tilde{\cD}_2$, $\delta \in (0,1)$. Then, there exists an absolute constant $C$, such that for any
\begin{equation*}
    L\geq C\left(m^4D^6\|f\|_{\tilde{\cD}_2}^5(\|f\|_{\tilde{\cD}_2}+D)^2\right)^{\frac{3}{\delta}},
\end{equation*}
there is an $L$-layer residual network $f_L(\cdot;\Theta)$ that satisfies
\begin{equation}
  \|f-f_L(\cdot;\Theta)\|^2\leq \frac{\|f\|_{\tilde{\cD}_2}^2}{L^{1-\delta}},
\end{equation}
and 
\begin{equation}
  \|\Theta\|_{\cP}\leq 9\|f\|_{\tilde{\cD}_1}
\end{equation}
\end{theorem}
}

\begin{theorem}[Inverse Approximation Theorem]\label{thm:inverse}
Let $f$ be a function defined on $X$. Assume that there is a sequence of residual networks $\{f_L(\cdot;\Theta_L)\}_{L=1}^\infty$ such that  $\|f(\bx)-f_L(\bx;\Theta_L)\|\rightarrow0$ as $L \rightarrow \infty$. Assume further that the parameters in $\{f_L(\cdot;\Theta)\}_{L=1}^\infty$ are (entry-wise) bounded by $c_0$. Then, we have $f\in\cD_\infty$, and 
$$\|f\|_{\cD_\infty}\leq\frac{2e^{m(c_0^2+1)}D^2c_0}{m}
$$
Moreover, if there exists constant $c_1$ such that $\|f_L\|_{\cD_1}\leq c_1$ holds for  any $L>0$, then we have 
$$\|f\|_{\cD_1}\leq c_1$$
\end{theorem}

The Rademacher complexity estimate is only established for  a family of modified flow-induced function norms $\|\cdot\|_{\hat{\cD}_p}$
(see the factor 2 in the definition below). It is not clear at this stage whether this is only a technical difficulty. 

Let
\begin{equation}\label{eqn: d_hat_norm}
 \|f\|_{\hat{\cD}_p} = \inf_{f=f_{\balpha,\{\rho_t\}}} |\balpha|^T \hat{N}_p(1)+\|\hat{N}_p(1)\|_1-D +\|\{\rho_t\}\|_\Lip, 
\end{equation}
where $\hat{N}_p(t)$ is given by
\begin{align*}
  \hat{N}_p(0) &= 2\bm{1}, \\
  \dot{\hat{N}}_p(t) &= 2\left( \EE_{\rho_t}(|\bU||\bW|)^p \right)^{1/p}\hat{N}_p(t).  
\end{align*}
Denote by $\hat{\cD}_p$ the space of functions with finite $\hat{\cD}_p$ norm. Then,  we have

\begin{theorem}[\cite{e2019barron}]\label{thm:rad_res}
Let $\hat{\cD}_p^Q=\{f\in\hat{\cD}_p: \|f\|_{\hat{\cD}_p}\leq Q\}$, then we have
\begin{equation}
\rad_n(\hat{\cD}_{2}^Q)\leq 18Q\sqrt{\frac{2\log(2d)}{n}}.
\end{equation}
\end{theorem}

Next we turn to the generalization error estimates for the regularized estimator.
At the moment, for the same reason as above,
such estimates have only been proved when the empirical risk is regularized by  a weighted path norm 
\begin{equation}\label{eqn: weighted_path}
\|\Theta\|_{\mathcal{WP}}=|\balpha|^T\prod\limits_{l=1}^L\left(I+\frac{2}{L}|\bU_l||\bW_l|\right)\bm{1},
\end{equation}
{which is the discrete version of~\eqref{eqn: d_hat_norm}.}
This norm assigns larger weights to paths that pass through more non-linearities. 
Now consider the residual network \eqref{eq:resnet} and the regularized empirical risk:
\begin{equation}
    \cJ(\Theta) :=
    \hat{\cR}(\Theta) + 3\lambda \|\Theta\|_{\mathcal{WP}}
    \sqrt{\frac{2 \log(2d)}{n}},
    \label{eqn: res-regularity}
\end{equation}

\begin{theorem}[\cite{ma2019priori}]\label{thm: res-apriori}
  Let $f^*: X \to [0, 1]$.    Fix any $\lambda \ge 4 + 2 / (3 \sqrt{2 \log(2d)})$. Assume that $\hat\Theta$ is an optimal solution of the regularized model~\eqref{eqn: res-regularity}. Then for any $\delta \in (0, 1)$, with probability at least $1 - \delta$ over the random training samples, the population risk satisfies
  \begin{equation}
    \cR(\hat\Theta) \le \frac{3 \|f\|_\cB^2}{L m}
    + (4 \|f\|_\cB + 1)\frac{3 (4 + \lambda) \sqrt{2 \log(2d)} + 2}{\sqrt n}
    + 4 \sqrt{\frac{2 \log(14 / \delta)}{n}}.
    \label{eqn:apriori}
  \end{equation}
\end{theorem}

\subsection{Multi-layer networks: Tree-like function spaces}

A neural network with $L$ hidden layers is a function of the form 
\begin{equation}\label{eq multi-layer network}
f(\bx) = \sum_{i_L=1}^{m_L}a^L_{i_L}\sigma\left(\sum_{i_{L-1}=1}^{m_{L-1}} a^{L-1}_{i_Li_{L-1}}\sigma\left( \dots \sigma\left(\sum_{i_1=1}^{m_1}a_{i_2i_1}^1\,\sigma\left(\sum_{i_0=1}^{d+1} a^0_{i_1i_0}\,x_{i_0}\right)\right)\right)\right)
\end{equation}
where $a^\ell_{i_{\ell+1}i_\ell}$ with $i_{\ell+1}\in [m_{\ell+1}]$ and $i_\ell \in [m_\ell]$ are the weights of the neural network. { Here the bias terms in the intermediate layers are omitted without loss of generality.} Analogous to the Barron norm, we  introduce the path-norm proxy by
\[
\sum_{i_L,\dots, i_0}\big|a^L_{i_L}\,a^{L-1}_{i_Li_{L-1}}{\cdots}\,a^1_{i_2i_1}\,a^0_{i_1i_0}\big|.
\]
and the path-norm of the function as the infimum of the path-norm proxies over all weights inducing the function.
\begin{equation}
\|f\|_{\W^L} = \inf\left\{\sum_{i_L,\dots, i_0}\big|a^L_{i_L}\,a^{L-1}_{i_Li_{L-1}}{\cdots}\,a^1_{i_2i_1}\,a^0_{i_1i_0}\big|\:\bigg|\: f \, \text{satisfies  \eqref{eq multi-layer network}}\right\}
\end{equation}
The natural function spaces for the purposes of approximation theory are the {\em tree-like function spaces} $\mathcal W^L$ of depth $L\in \N$ 
introduced in \cite{deep_barron}, where the norm can be extended in a natural way. For trees with a single hidden layer, Barron space and tree-like space agree, and the properties of tree-like function spaces are reminiscent of Barron space.  Instead of stating all the results as theorems, we will just list them below.

\begin{itemize}
\item Rademacher complexity/generalization gap: Let $\cF_Q = \{ f \in C^0, \|f\|_{\W^L} \le Q  \}$ be the closed ball of radius $Q>0$ in the tree-like function space. Then 
$\rad_S(\cF_Q) \leq 2^{L+1}Q \sqrt{\frac{2\ln(2d+2)}{n}}$. If a stronger version of the path-norm is controlled, then the dependence on depth can be weakened to $L^3$ instead of $2^L$ \cite{barron2018approximation}.
But remember, here we are thinking of keep $L$ fixed at some finite value and increase the widths of the layers.

\item The closed unit ball of the tree-like space of depth $L\geq 1$ is compact (in particular closed) in $C^0(K)$ for every compact set $K\subseteq \R^d$ and in $L^2(\P)$ for { every} compactly supported probability measure $\P$.

\item In particular, an inverse approximation theorem holds: If $\|f_m\|_{\W^L}\leq C$ and $f_m \to f$ in $L^2(P)$, then $f$ is in the tree-like function space of the same depth.

\item The direct approximation theorem holds, but not with Monte Carlo rate. For $f^*$ in a tree-like function space and $m\in\mathbb{N}$, there exists a network $f_m$ with layers of width $m_\ell = m^{L-\ell+1}$ such that 
\[
\|f_m-f^*\|_{L^2(P)} \leq \frac{2^L\,\|f^*\|_{\W^L}}{\sqrt{m}}.
\]
Note that the network has $O(m^{2L-1})$ weights. Part of this stems from the recursive definition, and by rearranging the index set into a tree-like structure, the number of free parameters (but dimension-independent) can be dropped to $O(m^{L+1})$. 

It is unclear whether a depth-independent (or less depth-dependent) approximation rate can be expected. Unlike two-layer networks and 
residual neural networks, layers of a multi-layer network discretize { into} conditional expectations, whereas the other function classes are naturally expressed as expectations. A larger number of parameters for multi-layer networks is therefore natural.

\item Tree-like function spaces form a scale in the sense that if $f$ is in the tree-like function space of depth $\ell$, then it is also in the tree-like function space of depth $L>\ell$ and $\|f\|_{\W^L}\leq 2\,\|f\|_{\W^\ell}$.

\item If $f:\R^d\to \R^k$ and $g:\R^k\to\R$ are tree-like functions in $\W^L$ and $\W^\ell$ respectively, then their composition $g\circ f$ is in $\W^{L+\ell}$.
\end{itemize}

In analogy to two-layer neural networks, we can prove a priori error estimates for multi-layer networks. Let $\P$ be a probability measure on $X$ and $S= \{\bx_1,\dots, \bx_N\}$ be a set of iid samples drawn from $\P$. %
For finite neural networks with weights $(a^L, \dots, a^0)\in \R^{m_L} \times \dots \times \R^{m_1\times d}$ we denote
\begin{align*}
\widehat \Risk_n (a^L,\dots,a^0) &= \hat{\Risk}_n (f_{a^L,\dots, a^0})\\
f_{a^L,\dots, a^0}(\bx) &= \sum_{i_L=1}^{m_L}a^L_{i_L}\sigma\left(\sum_{i_{L-1}=1}^{m_{L-1}} a^{L-1}_{i_Li_{L-1}}\sigma\left(\sum_{i_{L-2}} \dots \sigma\left(\sum_{i_1=1}^{m_1}a_{i_2i_1}^1\,\sigma\left(\sum_{i_0=1}^{d+1} a^0_{i_1i_0}\,x_{i_0}\right)\right)\right)\right).
\end{align*}

Consider the regularized loss function:
regularized risk functional
\[
L_n(a^0,\dots, a^L) = \widehat \Risk_n(a^L,\dots, a^0) + \frac{9\,L^2}m \left[\sum_{i_L=1}^{m_L}\dots \sum_{i_0=1}^{d+1} \big| a^L_{i_L}\,a^{L-1}_{i_Li_{L-1}}\dots a^0_{i_1i_0}\big|\right]^2
\]
\begin{theorem}[Generalization error]\label{theorem generalization error multi-layer}
Assume that the target function satisfies $f^*\in \W^L$. Let $\cH_m$ be the class of neural networks with architecture like in the direct approximation theorem for tree-like function spaces, i.e.\ $m_\ell = m^{L-\ell+1}$ for $\ell\geq 1$.
Let  $f_m$ be the function given by $\argmin_{(a^0,\dots, a^L) \in \cH_m} L_n$. Then $f_m$ satisfies the risk bound
\begin{equation}\label{eq a priori risk bound} 
\Risk(f_m) \leq \frac{18\,L^2\,\|f^*\|_{\W^L}^2}m +  2^{L+3/2}\|f^*\|_{\W^L}\sqrt{\frac{2\,\log(2d+2)}n} + \bar c\,\sqrt{\frac{2\,\log(2/\delta)}{n}}.
\end{equation}
\end{theorem}

In particular, there exists a function $f_m$ satisfying the risk estimate. Using a natural cut-off, the constant $\bar c$ can be replaced with $\|f^*\|_{L^\infty} \leq C\,\|f^*\|_{\W^L}$.

\subsection{Indexed representation and multi-layer spaces}\label{section multi-layer spaces}

Neural networks used in applications are not tree-like. To capture the structure of the functions represented by practical multi-layer neural
networks, \cite{deep_barron, nguyen2020rigorous} introduced an indexed representation of neural network functions.

\begin{definition}
For $0\leq i\leq L$, let $(\Omega_i, \A_i, \pi^i)$ be probability spaces where $\Omega_0 = \{0,\dots, d\}$ and $\pi^0$ is the normalized counting measure. Consider measurable functions $a^{L}:\Omega_L\to \R$ and $a^{i}:\Omega_{i+1}\times \Omega_i\to \R$ for $0\leq i \leq L-1$. The {\em arbitrarily wide neural network} modeled on the index spaces $\{\Omega_i\}$ with weight functions $\{a^i\}$ is 
{\footnotesize
\begin{equation}\label{eq network structure}
f_{a^L,\dots,a^0}(\bx) = \int_{\Omega_L} a^{(L)}_{\theta_L}\,\sigma\left(\int_{\Omega_{L-1}}\dots\sigma\left(\int_{\Omega_1} a^1_{\theta_2,\theta_1}\sigma\left(\int_{\Omega_0} a^0_{\theta_1,\theta_0}\,x_{\theta_0}\pi^0(\d\theta_0)\right)\pi^1(\d\theta_1)\right)\dots\,\pi^{(L-1)}(\d\theta_{L-1}) \right)\,\pi^L(\d\theta_L).
\end{equation}
}
\end{definition}

If the index spaces are finite, this coincides with finite neural networks and the integrals are replaced with finite sums. 
From this we see that the class of arbitrarily wide neural networks modeled on certain index spaces may not be a vector space. If all weight spaces $\{\Omega_i\}$ are sufficiently expressive (e.g.\ the unit interval with Lebesgue measure), then the set of multi-layer networks modeled on $\Omega_L,\dots,\Omega_0$ becomes a vector space. The key property is that $(0,1)$ can be decomposed into two sets of probability $1/2$, each of which is isomorphic to itself, so adding two neural networks of equal depth is possible.

The space of arbitrarily wide neural networks is a subspace of the tree-like function space of depth $L$ that
consists of functions $f$ with a finite path-norm 
\[
\|f\|_{\Omega_L,\dots,\Omega_0;K} = \inf \left\{\int_{\prod_{i=0}^L\Omega_i} \big| a^{(L)}_{\theta_L} \dots a^{(0)}_{\theta_1\theta_0}\big|\,\big(\pi^L\otimes\dots\otimes \pi^0\big)(\d\theta_L\otimes\dots\otimes \d\theta_0)\:\bigg|\: f= f_{a^L,\dots,a^0} \text{ on }K\right\}.
\]
Since the coefficients of non-consecutive layers do not share an index, this may be a proper subspace for networks with mutiple hidden layers. If $\Omega_i=(0,1)$, the space of arbitrarily wide networks with one hidden layer coincides with Barron space. 

In a network with two hidden layers, the (vector-valued) output of the zeroth layer is fixed independently of the second layer. Thus the output of the first layer is a vector whose coordinates lie in the subset of Barron space which have $L^1$-densities with respect to the (fixed) distribution of zeroth-layer weights on $\R^{d+1}$. This subspace is a separable subset of (non-separable) Barron space (see \cite{wojtowytsch2020representation} for some functional analytic considerations on Barron space). It is, however, still unclear whether the tree-like space and the space of arbitrarily wide neural networks on sufficiently expressive index spaces agree.
The latter contains all Barron functions and their compositions, including products of Barron functions. 

The space of measurable weight functions which render the path-norm finite is inconveniently large when considering training dynamics. 
To allow a natural gradient flow structure, we consider the subset of functions with $L^2$-weights. This is partially motivated by the observation that the $L^2$-norm of weights controls the path-norm.

\begin{lemma}
If $f= f_{a^L,\dots, a^0}$, then
\begin{align*}
\|f\|_{\Omega_L,\dots,\Omega_0;K} \leq \inf \left\{\|a^L\|_{L^2(\pi^L)}\,\prod_{i=0}^{L-1} \|a^i\|_{L^2(\pi^{i+1}\otimes\pi^i)} \:\bigg|\: a^i\text{ s.t. }f= f_{a^L,\dots,a^0} \text{ on }K\right\}
\end{align*}
\end{lemma}
\begin{proof}
{We sketch the proof for two hidden layers.
\begin{align*}
\int_{\Omega_2\times\Omega_1\times\Omega_0} \big|a^2_{\theta_2}& \,a^1_{\theta_2\theta_1}\,a^0_{\theta_1\theta_0}\big| \d\theta_2\,\d\theta_1\,\d\theta_0 
	= \int_{\Omega_2\times\Omega_1\times\Omega_0} \big|a^2_{\theta_2}a^0_{\theta_1\theta_0}\big| \,\big|a^1_{\theta_2\theta_1}\,\big| \d\theta_2\,\d\theta_1\,\d\theta_0\\
	&\leq \left(\int_{\Omega_2\times\Omega_1\times\Omega_0} \big|a^2_{\theta_2}a^0_{\theta_1\theta_0}\big|^2 \,\d\theta_2\,\d\theta_1\,\d\theta_0\right)^\frac12 \left(\int_{\Omega_2\times\Omega_1\times\Omega_0}  \big|a^1_{\theta_2\theta_1}\,\big|^2\,\d\theta_2\,\d\theta_1\,\d\theta_0\right)^\frac12\\
	&= \left(\int_{\Omega_2} \big|a^2_{\theta_2}\big|^2\,\d\theta_2\right)^\frac12 \left(\int_{\Omega_2\times\Omega_1} \big|a^1_{\theta_2\theta_1}\big|^2\d\theta_2\d\theta_1\right)^\frac12 \left(\int_{\Omega_1\times\Omega_0} \big|a^0_{\theta_1\theta_0}\big|^2\,\d\theta_1\d\theta_0\right)^\frac12.
\end{align*}
}
\end{proof}

Note that the proof is entirely specific to network-like architectures and does not generalize to tree-like structures. We define the measure of complexity of a function (which is not a norm) as
\[
Q(f) = \inf \left\{\|a^L\|_{L^2(\pi^L)}\,\prod_{i=0}^{L-1} \|a^i\|_{L^2(\pi^{i+1}\otimes\pi^i)} \:\bigg|\: a^i\text{ s.t. }f= f_{a^L,\dots,a^0} \text{ on }K\right\}.
\]

We can equip the class of neural networks modeled on index spaces $\{\Omega_i\}$ with a metric which respects the parameter structure.

\begin{remark}
The space of arbitrarily wide neural networks with $L^2$-weights can be metrized with the Hilbert-weight metric
\begin{align*} 
d_{HW}(f,g) = \inf\Bigg\{\sum_{\ell=0}^L \|a^{\ell, f} - a^{\ell, g}\|_{L^2(\pi^\ell)}\,\bigg|&\,a^{L, f},\dots, a^{0,g} \text{ s.t. }f = f_{a^{L,f}, \dots, a^{0, f}}, \:g = f_{a^{L,g},\dots, a^{0,g}}\text{ and}\\ \label{eq hilbert weight metric} \showlabel
	&\quad\|a^{\ell, h}\| \equiv \left(\prod_{i=0}^L\|a^{i,h}\|_{L^2}\right)^\frac1{L+1} \leq 2\,Q(h)^\frac1{L+1}\text{ for }h\in \{f,g\}\Bigg\}.
\end{align*}
The normalization across layers is required to ensure that functions in which one layer can be chosen identical do not have zero distance by shifting all weight to the one layer.

We refer to these metric spaces (metric vector spaces if $\Omega_i=(0,1)$ for all $i\geq 1$) as {\em multi-layer spaces}. They are complete metric spaces.
\end{remark}

\subsection{Depth separation in multi-layer networks}

We can ask how much larger $L$-layer space is compared to $(L-1)$-layer space. A satisfying answer to this question is still outstanding, but partial answers have been found, mostly concerning the differences between networks with one and two hidden layers.

\begin{example}
The structure theorem for Barron functions Theorem \ref{theorem structure barron} shows that functions which are non-differentiable on a curved hypersurface are not Barron. In particular, this includes distance functions from hypersurfaces like
\[
f(\bx) = \dist(\bx, S^{d-1}) = \big|1- \|\bx\|_{\ell^2}\big|.
\]
It is obvious, however, that $f$ is the composition of two Barron functions and therefore can be represented exactly by a neural network with two hidden layers. This criterion is easy to check in practice and therefore of greater potential impact than mere existence results.
But it says nothing about approximation by finite neural networks.
\end{example}

\begin{remark}\label{remark kolmogorov-arnold}

Neural networks with two hidden layers are significantly more flexible than networks with one hidden layer. In fact, there exists an activation function $\sigma:\R\to\R$ which is analytic, strictly monotone increasing and satisfies $\lim_{z\to\pm\infty}\sigma(z) = \pm 1$, but also has the surprising property that the finitely parametrized family
\[
\mathcal{H}= \left\{\sum_{i=1}^{3d}a_i\,\sigma\left(\sum_{j=1}^{3d} b_{ij}\,\sigma\left(\bw_j^T\bx\right)\right)\:\bigg|\: a_i, b_{ij}\in\R \right\}
\]
is dense in the space of continuous functions on any compact set $K\subset \R^d$ \cite{maiorov1999lower}. The proof is based on the Kolmogorov-Arnold representation theorem, and the function is constructed in such a way that the translations
\[
\frac{\sigma(z - 3m) + \lambda_m\,\sigma(z-(3m+1)) + \mu_m\,\sigma(z-(3m+2))}{\delta_m}
\]
for $m\in\N$ form a countable dense subset of $C^0[0,1]$ for suitable $\lambda_m, \mu_m, \delta_m\in\R$. In particular, the activation function is virtually impossible to use in practice. However, the result shows that any approximation-theoretic analysis must be specific to certain activation functions and that common regularity requirements are not sufficient to arrive at a unified theory.
\end{remark}

\begin{example}\label{example fourier non-approximable}
A separation result like this can also be obtained with standard activation functions. If $f$ is a Barron function, then there exists an $f_m(\bx) = \sum_{i=1}^m a_i\,\sigma(\bw_i^Tx)$ such that
\[
\|f- f_m\|_{L^2(P)} \leq \frac{C}{\sqrt m}.
\]
In \cite{eldan2016power}, the authors show that there exists a function $f$ such that
\[
\|f - f_m\|_{L^2(P)} \geq \bar c\,m^{-1/(d-1)}.
\]
but $f = g\circ h$ where $g,h$ are Barron functions (whose norm grows like a low degree polynomial in $d$). The argument is based on Parseval's identity and neighbourhood growth in high dimensions. Intuitively, the authors argue that $\|f_m - f\|_{L^2(\R^d)} = \|\hat f_m - \hat f\|_{L^2(\R^d)}$ and that the Fourier transform of $\sigma(\bw^T\bx)$ is concentrated on the line generated by $w$. If $\hat f$ is a radial function, its Fourier transform is radial as well and $f(\bx) = g(|\bx|)$ can be chosen such that $g$ is a Barron function and the Fourier transform of $f$ has significant $L^2$-mass in high frequencies. Since small neighborhoods $B_\eps(\bw_i)$ only have mass $\sim \eps^d$, we see that $\hat f$ and $\hat f_m$ cannot be close unless $m$ is very large in high dimension. To make this intuition precise, some technical arguments are required since $\bx\mapsto \sigma(\bw^T\bx)$ is not an $L^2$-function.
\end{example}

There are some results on functions which can be approximated better with significantly deeper networks than few hidden layers, but a systematic picture is still missing. To the best of our knowledge, there are no results for the separation between $L$ and $L+1$ hidden layers.

\subsection{Tradeoffs between learnability and approximation}

Example \ref{example fourier non-approximable} and Remark \ref{remark kolmogorov-arnold} can be used to establish more: If $\tilde f$ is a Barron function and $\|f-\tilde f\|_{L^2(P)} < \eps$, then there exists a dimension-dependent constant $c_d>0$ such that $\|\tilde f\|_{\mathcal B}\geq { c_d}\eps^{-\frac{d-3}2}$, i.e.\ $f$ cannot be approximated to high accuracy by functions of low Barron norm. To see this, choose $m$ such that $\frac{\bar c}4\,m^{-1/(d-1)} \leq \eps\leq \frac{\bar c}2\,m^{-1/(d-1)}$ and let $f_m$ be a network with $m$ neurons. Then
\[
2\eps \leq \bar c\, m^{-1/d} \leq  \|f_m - f\|_{L^2} \leq \|f_m - \tilde f\|_{L^2} +  \|\tilde f - f\|_{L^2}\leq \|f_m - \tilde f\|_{L^2}+ \eps,
\]
so if $f_m$ is a network which approximates the Barron function $\tilde f$, we see that
\[
\frac{\bar c}4\,m^{-1/(d-1)}\leq \eps \leq \|f_m - \tilde f\|_{L^2}\leq  \frac{\|\tilde f\|_{\mathcal B}} {m^{1/2}}.
\]
In particular, we conclude that
\[
\|\tilde f\|_{\mathcal B} \geq \frac{\bar c}4 m^{1/2-1/(d-1)} =\frac{\bar c}4 m^\frac{d-3}{2(d-1)} = \left(\frac 4{\bar c}\right)^{\frac{d-1}2} \left(\frac {\bar c}4\,m^{-1/(d-1)}\right)^{-\frac{d-3}2} \geq \left(\frac 4{\bar c}\right)^{\frac{d-1}2}\,\eps^{-\frac{d-3}2}.
\]

This is a typical phenomenon shared by {\em all} machine learning models of low complexity. Let $\P_d$ be Lebesgue-measure on the $d$-dimensional unit cube (which we take as the archetype of a truly `high-dimensional' data distribution).

\begin{theorem}\cite{approximationarticle}\label{theorem kolmogorov width decay}
Let $Z$ be a Banach space of functions such that the unit ball $B^Z$ in $Z$ satisfies
\[
\bE_{S\sim \P_d^n}\rad_S(B^Z) \leq \frac{C_d}{\sqrt{n}},
\]
i.e.\ the Rademacher complexity on a set of $N$ sample decays at the optimal rate in the number of data points. Then
\begin{enumerate}
\item The {\em Kolmogorov width} of $Z$ in the space of Lipschitz functions with respect to the $L^2$-metric is low in the sense that
\[
\limsup_{t\to \infty} \left[ t^\frac{2}{d-2} \sup_{f(0) =0,\:f\text{ is 1-Lipschitz}} \:\inf_{\|g\|_Z\leq t} \|f-g\|_{L^2(\P_d)}\right] \geq \bar c>0.
\]
\item There exists a function $f$ with Lispchitz constant $1$ such that $f(0)=0$, but
\[
\limsup_{t\to \infty} \left[t^\gamma\,\inf_{\|g\|_Z\leq t}\, \|f-g\|_{L^2(\P_d)}\right] = \infty\qquad\forall\ \gamma> \frac{2}{d-2}.
\]
\end{enumerate}
\end{theorem}

This resembles the result of \cite{maiorov1999best} for approximation by ridge functions under a constraint on the number of parameters, whereas here a complexity bound is assumed instead and no specific form of the model is prescribed (and the result thus applies to multi-layer networks as well).

Thus function spaces of low complexity are `poor approximators' for general classes like Lipschitz functions since we need functions of large $Z$-norm to approximate functions to a prescribed level of accuracy. This includes all function spaces discussed in this review, although some spaces are significantly larger than others (e.g.\ there is a large gap between reproducing kernel Hilbert spaces, Barron space, and tree-like three layer space).

\subsection{A priori vs. a posteriori estimates}

The error estimate given above should be compared with a more typical form of estimate in the  machine learning literature: 
\begin{equation}
\mathcal{R}(\hat{\theta}_n)- \hat{\mathcal{R}}_n(\hat{\theta}_n) 
     \lesssim \frac{\|\hat{\theta}_n \|} {\sqrt{n}}
     \label{aposteriori}
\end{equation}
where $\|\hat{\theta}_n \|$ is some suitably defined norm. 
Aside from the fact that \eqref{apriori-2layer} gives a bound on the total generalization error and 
\eqref{aposteriori} gives a bound on the generalization gap, there is an additional important difference:
The right hand side of \eqref{apriori-2layer} depends only on the target function $f^*$, not the output of the machine learning model.
The right hand side of \eqref{aposteriori} depends only on the output of the machine learning model, not the target function.
In accordance with the practice in finite element methods, we call \eqref{apriori-2layer} a priori estimates and
\eqref{aposteriori} a posteriori estimates.

{How good and how useful are these estimates?}
A priori estimates discussed here tell us in particular that there exist functions in the hypothesis space for which the generalization error does 
not suffer from the CoD if the target function lies in the appropriate function space.
It is likely that these estimates are nearly optimal in the sense that they are comparable to Monte Carlo error rates
(except for multi-layer neural networks, see below).
It is possible to improve these estimates, for example using standard tricks for Monte Carlo sampling for the approximation error
and local Rademacher complexity \cite{bartlett2005local} for the estimation error .
However, these would only improve the exponents in $m$ and $n$ by $O(1/d)$ which diminishes for large $d$.

Regarding the quantitative value of these a priori estimates, the situation is less satisfactory. 
The first issue is that the values of the norms  are not known since the target function is not known.
One can estimate these values using the output of the machine learning model, but this does not give us  rigorous bounds.
An added difficulty is that the norms are defined as an infimum over all possible representations, the output
of the machine learning model only gives one representation. 
But even if we use the exact values of these norms, the bounds given above are still not tight.
For one thing, the use of Monte Carlo sampling to control the approximation error does not give a tight bound.
This by itself is an interesting issue.

The obvious advantage of the a posteriori bounds is that they can be readily evaluated and give us quantitative bounds for the
size of the generalization gap.  Unfortunately this has not been borned out in practice: The values of these norms are so enormous
that these bounds are almost always vacuous \cite{dziugaite2017computing}.

In finite element methods, a posteriori estimates are used to help refining the mesh in adaptive methods.
Ideally one would like to do the same for machine learning models. However, little has been done in this direction.

 Since the a posteriori bounds only controls the generalization gap, not the full generalization
error, it misses an important aspect of the whole picture, namely, the approximation error.
In fact, by choosing a very strong norm, one can always obtain estimates of the type in \eqref{aposteriori}.
However, with such strongly constrained hypothesis space, the approximation error might be huge.  
This is indeed the case for some of
the norm-based a posteriori estimates in the literature.  See  \cite{ma2019priori} for  examples.

\subsection{What's not known?}

Here is a list of problems that we feel are most pressing.

1. Sharper estimates.  There are two  obvious places where one should be able to improve the estimates.

\begin{itemize}
\item  In the current analysis, the approximation error is estimated with the help of Monte Carlo sampling.
This gives us the typical size of the error for randomly picked parameters.
However, in machine learning, we are only interested in  the smallest  error. This is a clean mathematical problem that has not
received attention.

\item  The use of Rademacher complexity to bound the generalization gap neglects the fact that the integrand
in the definition of the population risk (as well as the empirical risk) should itself be small, since it is the point-wise error.
This should be explored further.  We refer to \cite{bartlett2005local} for some results in this direction.\
\end{itemize}

2.   The rate for the approximation error for functions in multi-layer spaces is not the same as Monte Carlo. Can this be improved? 

More generally, it is not clear whether the multi-layer spaces defined earlier are the right spaces for multi-layer neural networks.

3.  We introduced two kinds of flow-induced spaces for residual neural networks.  To control the Rademacher complexity,
we had to introduce the weighted path norm.  This is not necessary for the approximation theory.
It is not clear whether similar Rademacher complexity estimates can be proved for the spaces $\mathcal{D}_2$ or $\tilde{\mathcal{D}_2}$.

4.  Function space for convolutional neural networks that fully explores the benefit of symmetry.

5.  Another interesting network structure is the DenseNet \cite{huang2017densely}. Naturally one is interested in the natural function space
associated with DenseNets.

\section{The loss function and the loss landscape}\label{section landscape}

It is a surprise to many that simple gradient descent algorithms work quite well for optimizing the loss function in common
machine learning models.
To put things into perspective, no one would dream of using the same kind of algorithms for protein folding -- the energy
landscape for am typical protein is so complicated with lots of local minima that gradient descent algorithms will not go very far.
The fact that they seem to work well for machine learning models strongly suggests that the landscapes for the loss functions
 are qualitatively different.
An important question is  to quantify exactly how the loss landscape looks like.
Unfortunately, theoretical results on this important problem is still quite scattered and there is not yet a general 
picture that has emerged. But generally speaking, the current understanding is that while it is possible to find arbitrarily bad examples
for finite sized neural networks,  their landscape simplifies as the size increases.

To begin with, the loss function is non-convex and it is easy to cook up models for which the loss landscape has bad local minima
(see for example \cite{swirszcz2016local}).
Moreover, it has been suggested that for small  size two-layer ReLU networks with teacher networks as the target function, 
nearly all target networks lead to spurious local minima, and 
 the probability of hitting such local minima is quite high \cite{safran2018spurious}. 
 It has also been suggested the over-parametrization helps to avoid these bad spurious local minima.
 
 The loss landscape of large networks can be very complicated.  \cite{skorokhodov2019loss} presented some amusing numerical results 
 in which the authors demonstrated that one can find arbitrarily complex patterns near the global minima of the loss function.
 Some theoretical results along this direction were proved in \cite{czarnecki2019deep}. Roughly speaking, it was shown that for any $\veps > 0$,
 every low-dimensional pattern can be found
in a loss surface of a sufficiently deep neural network, and
within the pattern there exists a point whose loss
is within $\veps$ of the global minimum \cite{czarnecki2019deep}.

On the positive side,  a lot is known for linear and quadratic neural network models.
For linear neural network models,  it has been shown that \cite{kawaguchi2016deep}:
(1) every local minimum is
a global minimum, (2) every critical point that is not a global minimum is a saddle
point,  (3) for  networks with more than three layers there exist ``bad'' saddle points where the Hessian has no negative
eigenvalue and (4)  there
are  no such bad saddle points for networks with three layers.

Similar results have been obtained for over-parametrized two-layer neural network models with quadratic
activation \cite{soltanolkotabi2018theoretical, du2018power}.  In this case it has been shown under various conditions that (1) all local minima are global
and (2) all saddle points are strict, namely there are directions of strictly negative curvature.
Another interesting work for the case of quadratic activation function  is \cite{mannelli2020optimization}.
In the case when the target function is a ``single neuron'',   \cite{mannelli2020optimization} gives an asymptotically exact (as $d \rightarrow \infty$) characterization of the number of
training data samples needed for the global minima of the empirical loss to give rise to a unique function, namely the target function.

For over-parametrized neural networks with smooth activation function, the structure of the global minima is characterized in the paper of Cooper \cite{cooper2018loss}:
Cooper proved that the locus of the global minima is generically (i.e. possibly after an arbitrarily small change to the data set) a smooth 
$m-n$  dimensional submanifold of $\R^{m}$ where $m$ is the number of free parameters in the neural network model and $n$ is the 
training data size. 

The set of minimizers of a convex function is convex, so unless the manifold of minimizers is an affine subspace of $\R^m$, the loss function is non-convex (as can be seen by convergence to poor local minimizers from badly chosen initial conditions). If there are two sets of weights such that both achieve minimal loss, but not all weights on the line segment between them do, then somewhere on the connecting line there exists a local maximum of the loss function (restricted to the line). In particular, if the manifold of minimizers is curved at a point, then arbitrarily closedby there exists a point where the Hessian of the loss function has a negative eigenvalue. This lack of positive definiteness in training neural networks is observed in numerical experiments as well. The curvature of the set of minimizers partially explains why the weights of neural networks converge to different global minimizers depending on the initial condition.

For networks where half of the weights in every layer can be set to zero if the remaining weights are rescaled appropriately, the set of global minimizers is connected \cite{kuditipudi2019explaining}.

We also mention the interesting empirical work reported in \cite{li2018visualizing}.  Among other things,
it was demonstrated that adding skip connection has a drastic effect on smoothing the loss landscape.

\subsection{What's not known?}

We still lack a good mathematical tool to describe the landscape of the loss function
for large neural networks. In particular, are there local minima and how large is the basin of attraction
of these local minima if they do exist?

For fixed, finite dimensional gradient flows,
knowledge about the landscape allows us to draw conclusions about the qualitative behavior of  the gradient descent dynamics independent 
of the detailed dynamics.  In machine learning, the dimensionality of the loss function is $m$, the number of free parameters, and
we are interested in the limit as $m$ goes to infinity.  So it is tempting to ask about the landscape of the limiting (infinite dimensional) problem.
It is not clear whether this can be formulated as a well-posed mathematical problem.

\section{ The training process: convergence and implicit regularization}\label{section training}

The results of Section 3 tell us that good solutions do exist in the hypothesis space.
The amazing thing is that simple minded gradient descent algorithms are able to find them, even though
one might have to be pretty good at parameter tuning.
In comparison, one would never dream of using gradient descent to perform protein folding, since the landscape of protein
folding is so complicated with lots of bad local minima. 

The basic questions about the training process are:
\bi
\item Optimization: Does the training process converge to a good solution?  How fast?
\item Generalization:  Does the solution selected by the training process generalize well?
 In particular, is there such thing as ``implicit regularization''? What is the mechanism for such  implicit regularization?
\ei

At the moment, we are still quite far from being able to answering these questions completely,  but an intuitive picture 
has started to emerge.

We will mostly focus on the gradient descent (GD) training dynamics.
But we will touch upon some important qualitative features of other training algorithms such as 
stochastic gradient descent (SGD) and Adam.


\subsection{Two-layer neural networks with mean-field scaling}

``Mean-field'' is a notion in statistical physics that describes a particular form of interaction
between particles.  In the mean-field situation, particles interact with each other only through a mean-field which
every particle contributes to more or less equally.
The most elegant mean-field picture in machine learning is found in the case of two-layer neural networks:
If one views the neurons as interacting particles, then these particles only interact with each other through the
function represented by the neural network, the mean-field in this case. This observation was first made in 
\cite{chizat2018global, mei2018mean, rotskoff2018parameters,sirignano2018mean}. By taking the hydrodynamic limit for the gradient flow of  finite neuron systems, 
 these authors obtained a continuous integral differential
equation that describes the evolution of the probability measure for the weights associated with the neurons.

Let 
$$
I(\bu_1, \cdots, \bu_m) = \hat{\mathcal{R}}_n(f_m), \quad \bu_j = (a_j, \bw_j),
\quad  f_m(\bx) =
\frac 1 m \sum_j  a_j \sigma(\wb_j^T \xb) 
$$
Define the GD dynamics  by:
\begin{equation}
\label{GD-NN}
\frac{d \bu_j}{dt} = - { m}\nabla_{\bu_j} I(\bu_1, \cdots, \bu_m), \quad
\bu_j(0) = \bu_j^0, \quad j \in [m]
\end{equation}

\begin{lemma}Let
$$ 
\rho(d\bu, t) =  \frac 1 m \sum_j \delta_{\bu_j(t)}
$$
then the GD dynamics \eqref{GD-NN} can be expressed equivalently as:
\begin{equation}
\partial_{t} \rho = \nabla (\rho \nabla V), \quad V = \frac{\delta \hat{\mathcal{R}}_n}{\delta \rho}
\label{GD-cont-NN}
\end{equation}
\end{lemma}


\eqref{GD-cont-NN} is the mean-field equation that describes the evolution of the probability distribution for the weights associated with 
each neuron.
The lemma above simply states that  \eqref{GD-cont-NN} is satisfied for  finite neuron systems.

It is well-known that  \eqref{GD-cont-NN} is the gradient flow of $\hat{\mathcal{R}}_n$ under the Wasserstein metric.
This brings the hope that the mathematical tools developed in the theory of optimal transport can be brought to bear for the analysis of
\eqref{GD-cont-NN} \cite{villani2008optimal}.  In particular, we would like to use these tools to study the qualitative behavior of
the solutions of \eqref{GD-cont-NN} as $ t \rightarrow \infty $.
Unfortunately the most straightforward application of the results from optimal transport theory requires that the risk functional
be displacement convex \cite{mccann1997convexity}, a property that rarely holds in machine learning (see however the example in \cite{javanmard2019analysis}).
As a result, less than expected has been obtained using optimal transport theory.

The one important result, due originally to Chizat and Bach \cite{chizat2018global}, is the following.
We will state the result for the population risk.

\begin{theorem} \cite{chizat2018global, chizat2020implicit, relutraining}\label{theorem convergence Chizat Bach}
Let $\{\rho_t\}$ be a solution of the Wasserstein gradient flow such that
\begin{itemize}
\item $\rho_0$ is a probability distribution on the cone $\Theta:= \{|a|^2 \leq |w|^2\}$.
\item Every open cone in $\Theta$ has positive measure with respect to $\rho_0$.
\end{itemize}
Then the following are equivalent.
\begin{enumerate}
\item The velocity potentials $\frac{\delta \mathcal{R}} {\delta\rho}(\rho_t,\cdot)$ converge to a unique limit as $t\to\infty$.
\item $\mathcal{R}(\rho_t)$ decays to the global infimum value as $t\to\infty$. 
\end{enumerate}
If either condition is met, the unique limit of $\mathcal{R}(\rho_t)$ is zero. If  $\rho_t$ also converges in the
Wasserstein metric, then the limit $\rho_\infty$ is a minimizer.
\end{theorem}

Intuitively, the theorem is slightly stronger than the statement that $\mathcal{R}(\rho_t)$ converges to its infimum value if and only if its derivative \ converges to zero in a suitable sense (if we approach from  a flow of omni-directional distributions). The theorem is more a statement about a training algorithm than the energy landscape, and specific PDE arguments are used in its proof.  A few remarks are in order:

\begin{enumerate}
\item There are further technical conditions for the theorem to hold.
\item Convergence of subsequences of $\frac{\delta \mathcal{R}} {\delta\rho}(\rho_t,\cdot)$ is guaranteed by compactness. The question of whether they converge to a unique limit can be asked independently of the initial distribution and therefore may be more approachable by standard means.
\item The first assumption on $\rho_0$ is a smoothness assumption needed for the existence of the gradient flow.
\item The second assumption on $\rho_0$ is called {\em omni-directionality}. It ensures that $\rho$ can shift mass in any direction which reduces risk. The cone formulation is useful due to the homogeneity of ReLU.
\end{enumerate}

This is almost the only non-trivial rigorous result known on the 
global convergence of gradient flows in the nonlinear regime.
In addition, it reveals the fact that having full support for the probability distribution (or something to that effect)
 is an important property that helps for the global convergence.

The result is insensitive to whether a minimizer of the risk functional exists (i.e. whether the target function is in Barron space). If the target function is not in Barron space, convergence may be very slow since the Barron norm increases sub-linearly during gradient descent training.

\begin{lemma}\label{lemma sublinear growth}\cite[Lemma 3.3]{relutraining}
If $\{\rho_t\}$ evolves by the Wasserstein-gradient flow of $\Risk$, then
\[
\lim_{t\to \infty} \frac{\|f_{\rho_t}\|_{\mathcal B}}t = 0.
\]
\end{lemma}

In high dimensions, even reasonably smooth functions are difficult to approximate with functions of low Barron norm in the sense of Theorem \ref{theorem kolmogorov width decay}. Thus a ``dynamic curse of dimensionality'' may affect gradient descent training if the target function is not in Barron space. 

\begin{theorem}\label{main theorem 2}
Consider population and empirical risk expressed by the functionals 
\[
\Risk(\rho) = \frac12 \int_{X} (f_\rho- f^*)^2(\bx) d\bx, \qquad \hat{\Risk}_n(\rho) = \frac1{2n} \sum_{i=1}^n   (f_\rho- f^*)^2(\bx_i)
\]
where the points $\{\bx_i\}$ are i.i.d samples from the uniform distribution on $X$. 
There exists $f^*$ with Lipschitz constant and $L^\infty$-norm bounded by $1$
such that the parameter measures $\{\rho_t\}$ defined by the $2$-Wasserstein gradient flow of either $\hat{\Risk}_n$ or $\Risk$ satisfy
\[
\limsup_{t\to \infty} \big[t^\gamma\,\Risk(\rho_t)\big] = \infty
\]
for all $\gamma> \frac{4}{d-2}$. 
\end{theorem}




\begin{figure}[!h]
    \centering
    \includegraphics[width=\textwidth]{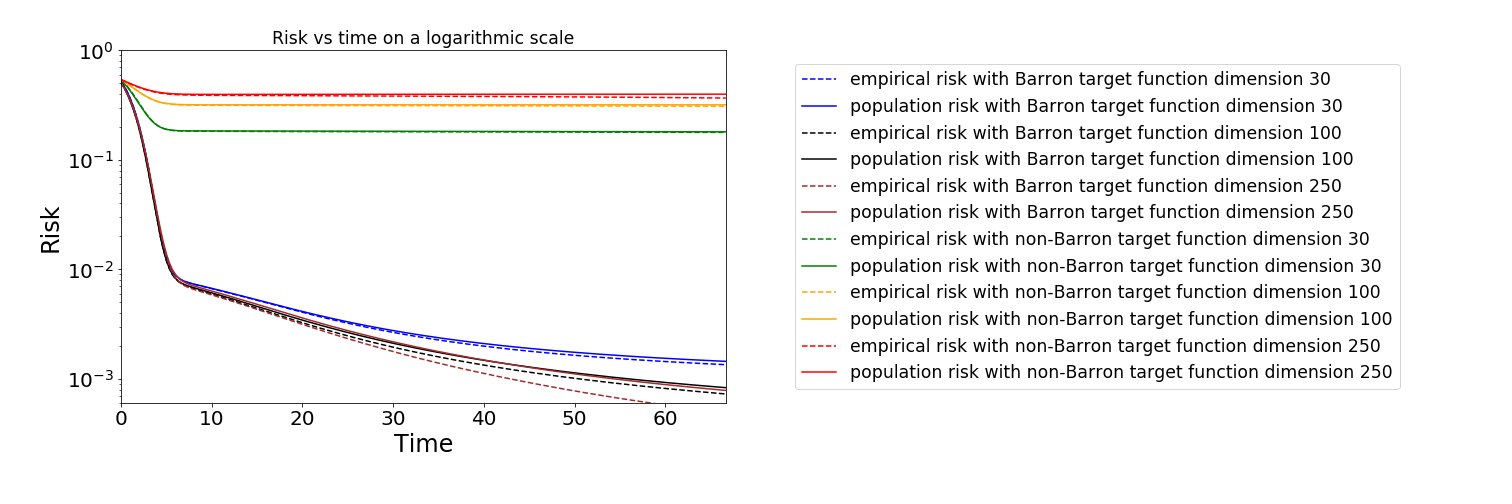}
    \vspace{-8mm}
    \caption{\small The rate of convergence of the gradient descent dynamics for Barron and non-Barron functions { on a logarithmic scale}. 
    Convergence rates for Barron functions seem to be dimension-independent. This is not the case for non-Barron functions.}
  \label{rate-Barron}
\end{figure}

\subsection{Two-layer neural networks with conventional scaling}

In practice, people often use the
 conventional scaling (instead of the mean-field scaling) which takes the form:
 
 $$ f_m(\bx;\ba,\bB) = \sum_{j=1}^m a_j \sigma(\bb_j^T\bx) = \ba^T\sigma(\bB \bx), $$
A popular initialization \cite{lecun2012efficient,He2015b} is as follows
$$  a_j(0) \sim \mathcal{N}(0,\beta^2), \qquad \bb_j(0)\sim \mathcal{N}(0,I/d) $$
where $\beta =0 $ or  $1/\sqrt{m}$.
For later use, we define the Gram matrix $K = (K_{ij})\in\RR^{n\times n}$:
\[
    K_{i,j} = \frac{1}{n}\EE_{\bb\sim\pi_0}[ \sigma(\bb^T\bx_i)\sigma(\bb^T\bx_j)].
\]

With this ``scaling'', the one case where a lot is known is the so-called highly over-parametrized regime.
There is both good and bad news in this regime.
The good news is that one can prove exponential convergence  to global minima of the empirical risk.

\begin{theorem} [\cite{du2018gradient}]\label{thm: du}
Let $\lambda_n=\lambda_{\min}(K)$ and assume $\beta=0$. For any $\delta\in (0,1)$, assume that $m\gtrsim n^2\lambda_n^{-4}\delta^{-1}\ln(n^2\delta^{-1})$. Then with probability at least $1-6 \delta$ we have 
\begin{equation}
\begin{aligned}
\hat{\cR}_n(\ba(t),\bB(t))&\leq e^{-m\lambda_n t}\hat{\cR}_n(\ba(0), \bB(0))
\end{aligned}
\end{equation}
\end{theorem}

Now the bad news: the generalization property of the converged solution is no better than that of the associated random feature model,
 defined by { freezing}
$\{ \bb_j \}= \{\bb_j(0) \}$,  and only training $\{a_i \}$.

The first piece of insight that the underlying dynamics in this regime is effectively linear is given in  \cite{daniely2017sgd}. 
\cite{jacot2018neural}   termed the effective kernel the ``neural tangent kernel''.
Later it was proved rigorously that in this regime, the entire GD path for the two-layer neural network model is
uniformly close to that of the associated random feature model  \cite{ma2019comparative, arora2019exact}.

\begin{theorem}[\cite{ma2019comparative}]
\label{rfm}
{ Let $\bB_0=\bB(0)$. Denote  by $f_m(\bx;\tilde{\ba}(\cdot), \bB_0)) $ the solutions of GD dynamics for the random feature model. }
Under the same setting as Theorem \ref{thm: du}, we have
\begin{equation}
\sup_{\bx\in S^{d-1}}    |f_m(\bx;\ba(t), \bB(t)) - f_m(\bx;\tilde{\ba}(t),\bB_0)| \lesssim \frac{(1+\sqrt{\ln(1/\delta)})^2\lambda_n^{-1}}{\sqrt{m}}.
\end{equation}
\end{theorem}

In particular, there is no ``implicit regularization'' in  this regime.

 Even though convergence of the training process can be proved, overall this is a disappointing result.
 At the theoretical level, it does not shed any light on possible implicit regularization. 
 At the practical level, it tells us that high over-parametrization is indeed a bad thing for these models.
 
 What happens in practice? Do less over-parametrized regimes exist for which implicit regularization does actually  happen?
 Some insight has been gained from the numerical study in \cite{ma2020quenching}.
  
Let us look at a simple example:  the single neuron target function:
$f_1^*(\bx) = \sigma(\bb^*\cdot \bx), \quad  \bb^*=\bm{e}_1$.  
This is admittedly a very simple target function.  Nevertheless, the training dynamics for this function is
quite representative of 
 target functions which can be accurately approximated by a small number of neurons (i.e. effectively ``over-parametrized'').

\begin{figure}[!h]
    \centering
    \begin{subfigure}{0.4\textwidth}
    \includegraphics[width=\textwidth]{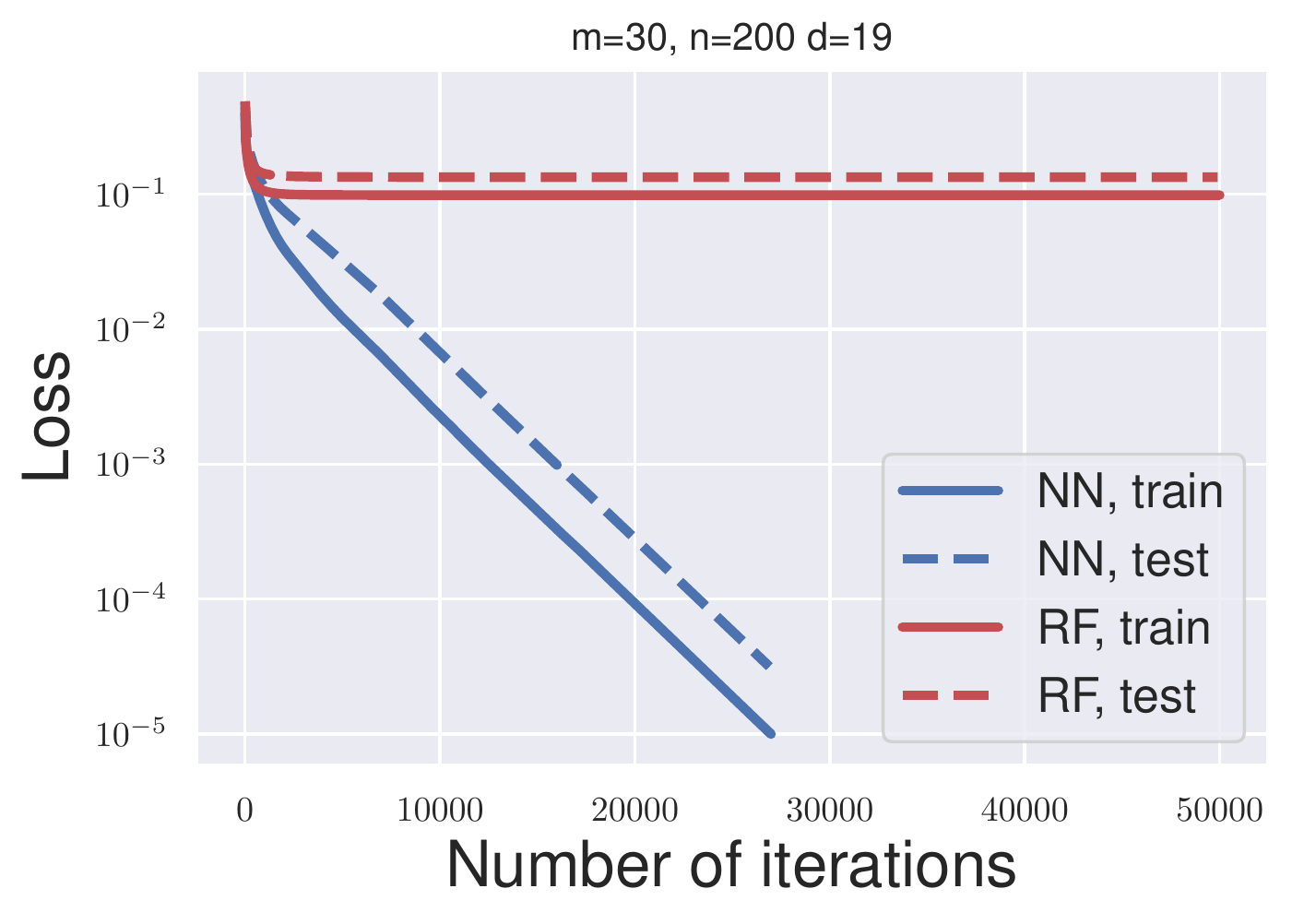}
    \caption{}
    \label{}
    \end{subfigure}
    \begin{subfigure}{0.4\textwidth}
    \includegraphics[width=\textwidth]{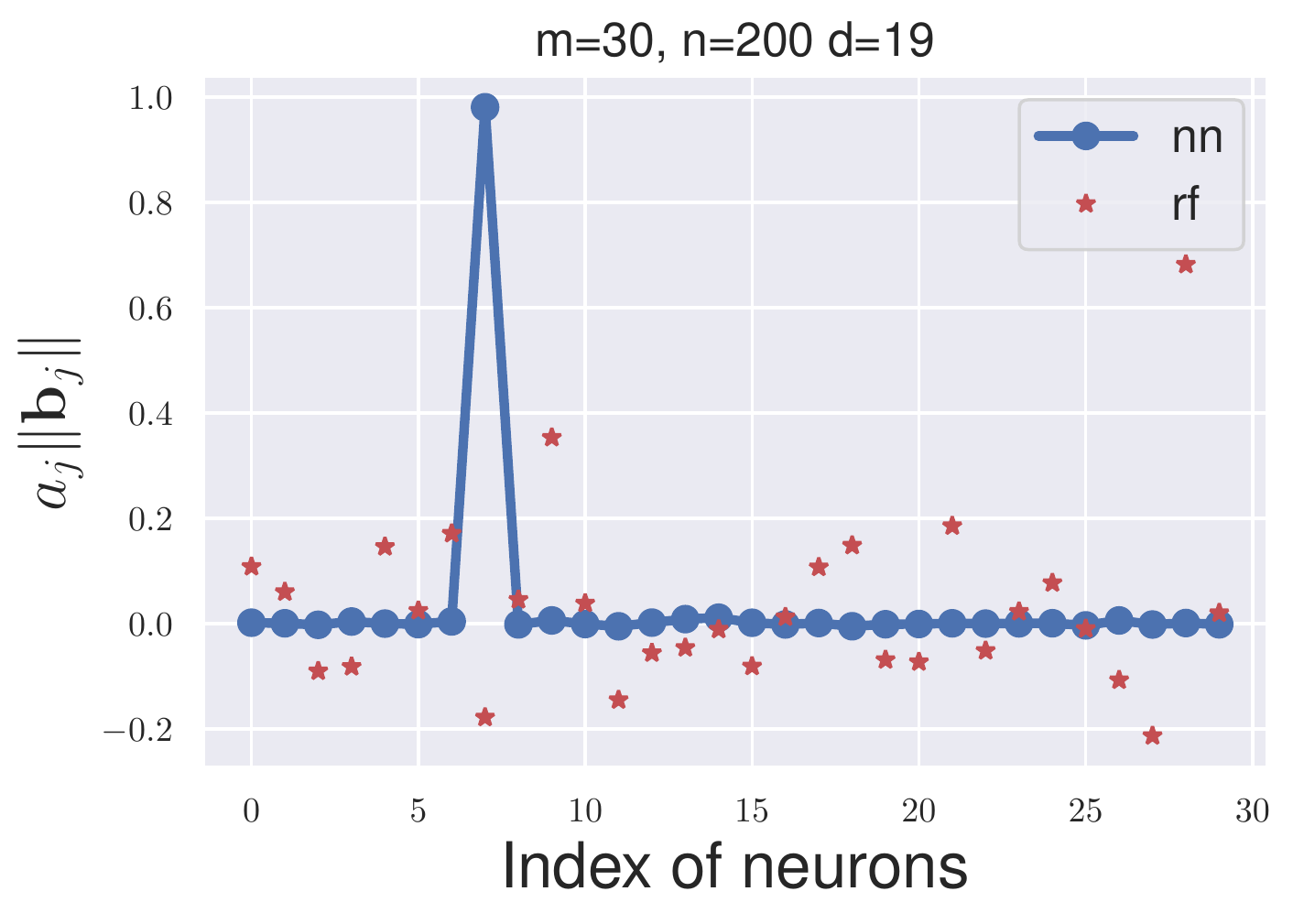}
    \caption{} 
    \label{}
    \end{subfigure}
    \begin{subfigure}{0.4\textwidth}
    \includegraphics[width=\textwidth]{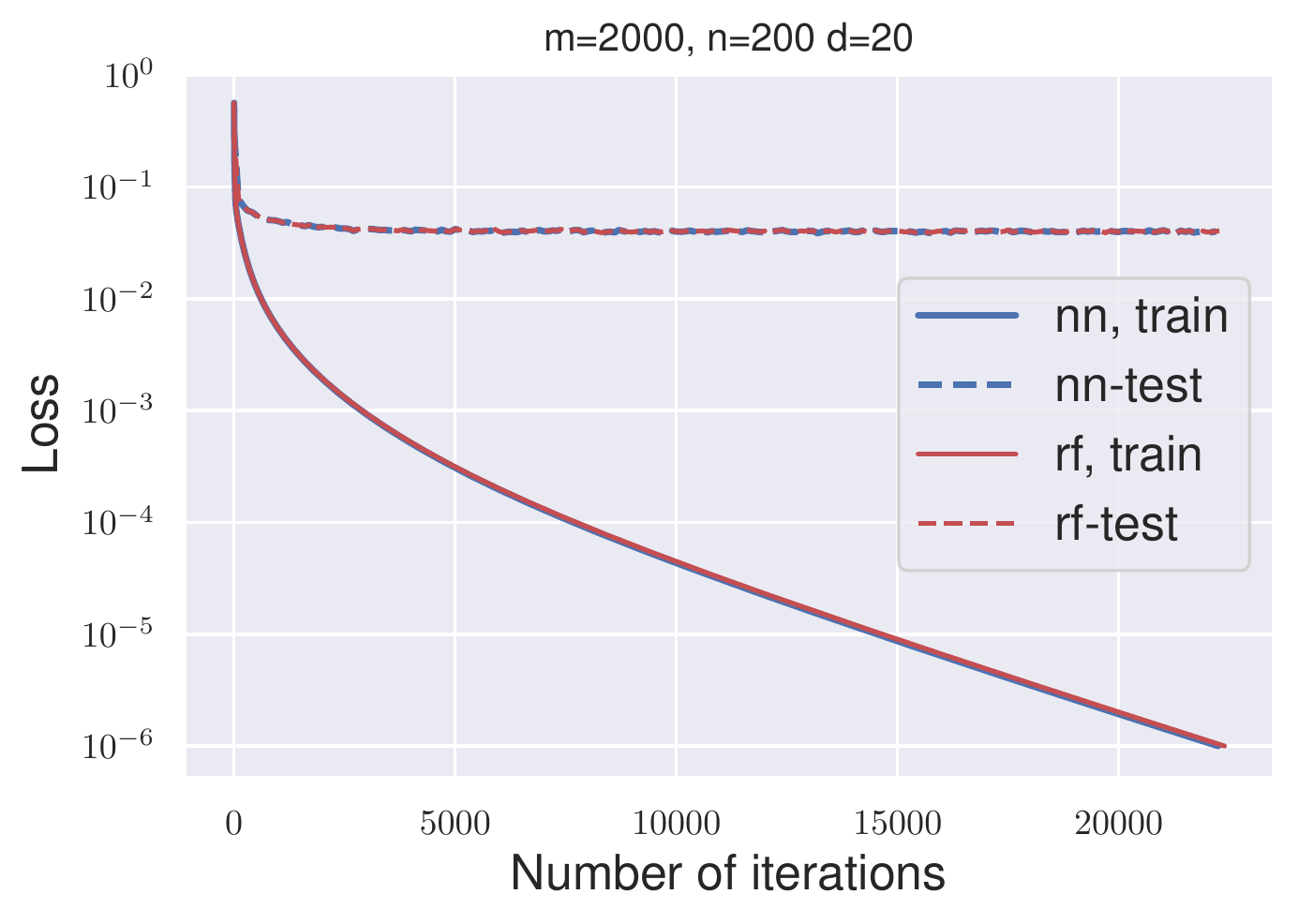}
    \caption{}\label{}
    \end{subfigure}
    \begin{subfigure}{0.4\textwidth}
    \includegraphics[width=\textwidth]{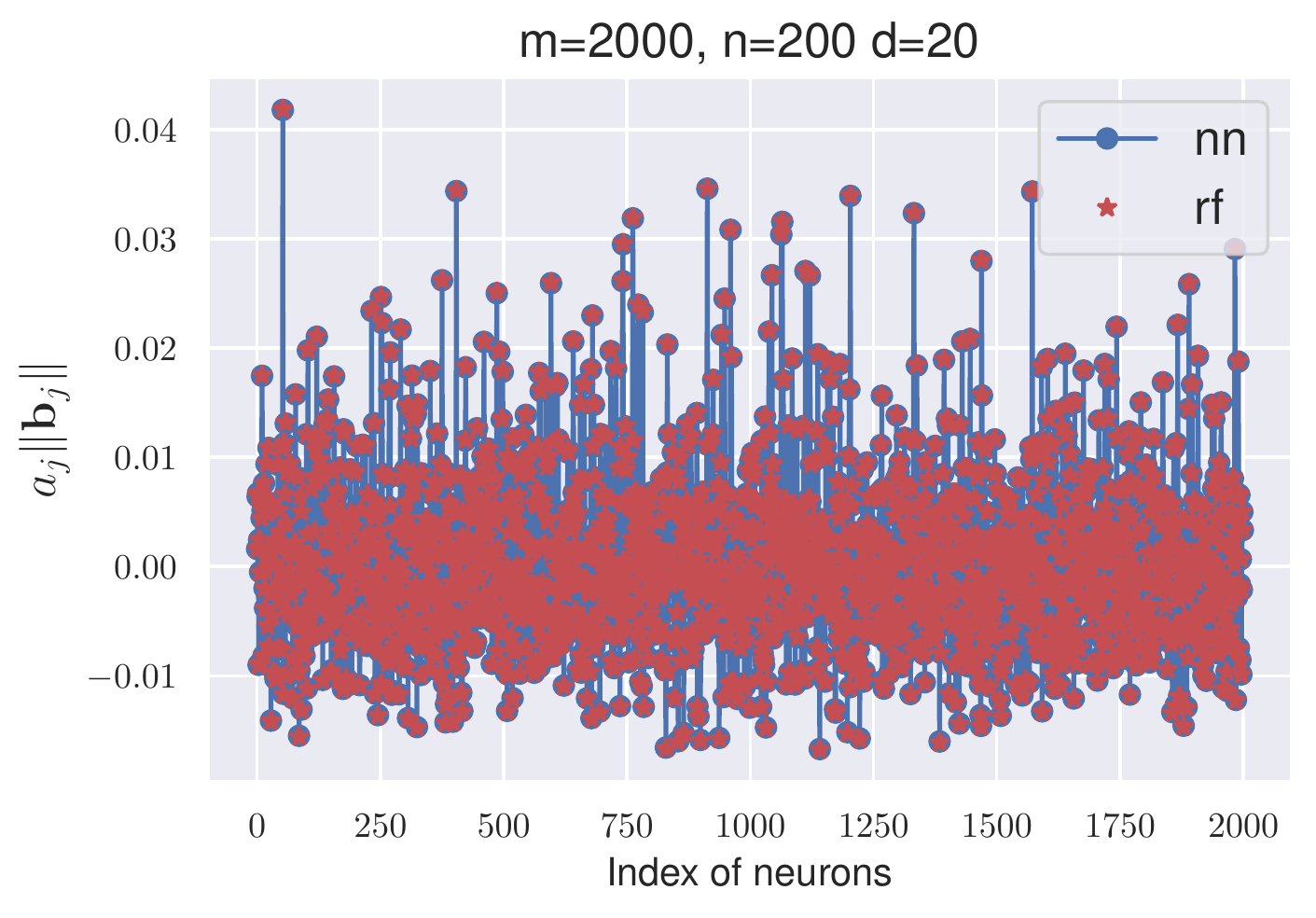}
    \caption{}\label{}
    \end{subfigure}
     \vspace{-2mm}
    \caption{\small The dynamic behavior of learning single-neuron target function using GD with conventional scaling. We also show the results of the random feature model as a comparison. \textbf{Top:} the case when $m=30, n=200, d=19$ in the mildly over-parametrized regime. \textbf{Bottom:} the case when $m=2000, n=200, d=19$ in the highly over-parametrized regime.
     The learning rate $\eta=0.001$.
    (a,c) The dynamic behavior of the training and test error. (b,d) The outer layer coefficient of each neuron for the converged solution.     \label{dd}
    }
    \end{figure}

    Figure \ref{dd} shows some results from \cite{ma2020quenching}.
    First note that the bottom two figures represent exactly the kind of results stated in Theorem \ref{rfm}.
    The top two figures suggest that the training dynamics displays two phases.  In the first phase, the training dynamics
    follows closely that of the associated random feature model.
    This phase quickly saturates.  The training dynamics 
    for the neural network model is able to reduce the training (and test) error further in the second phase 
    through a quenching-activation process:
    Most neurons are quenched in the sense that their outer layer coefficients $a_j$ become very small,
    with the exception of only a few activated ones.

    \begin{figure}[!ht]
    \centering
    \includegraphics[width=0.425\textwidth]{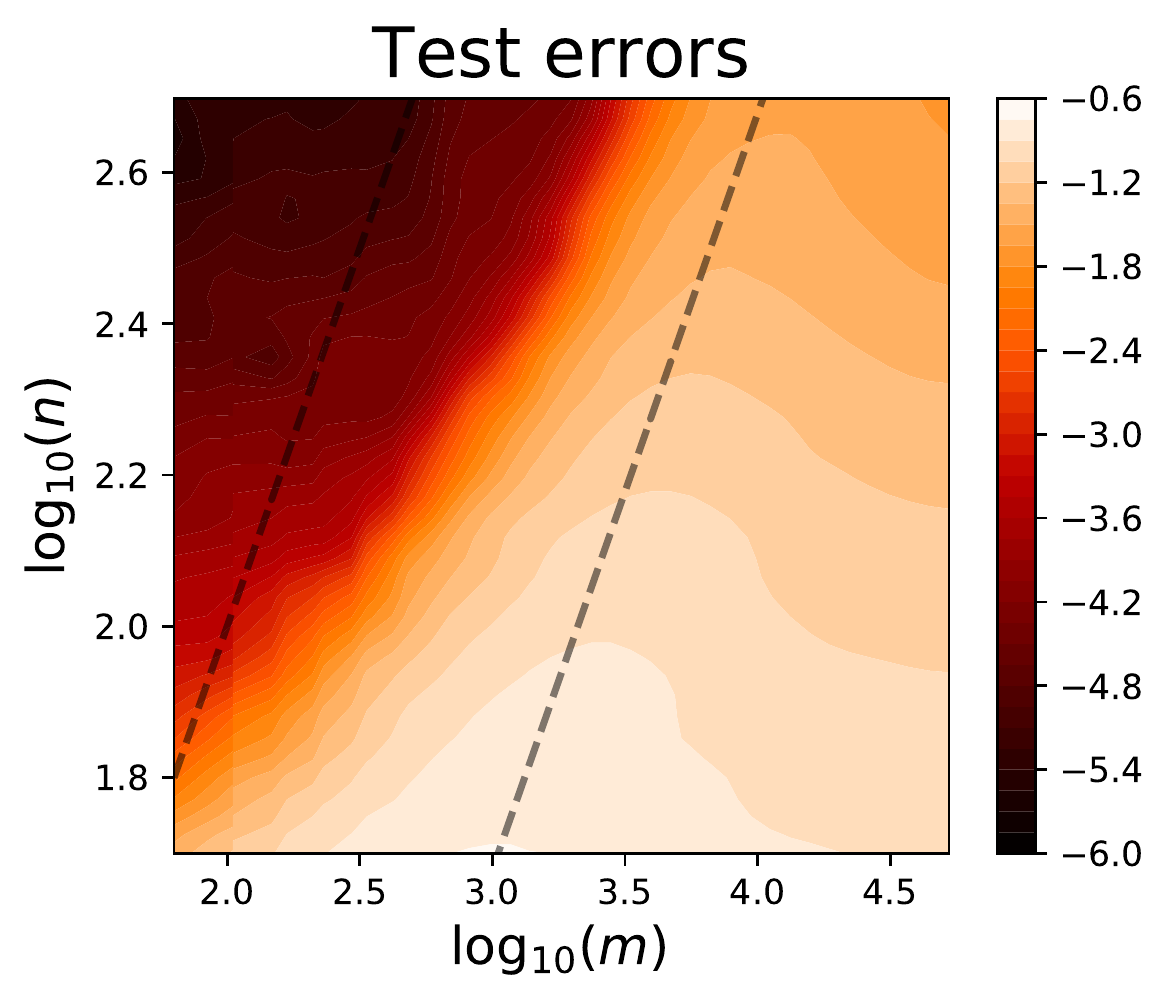}
     \includegraphics[width=0.37\textwidth]{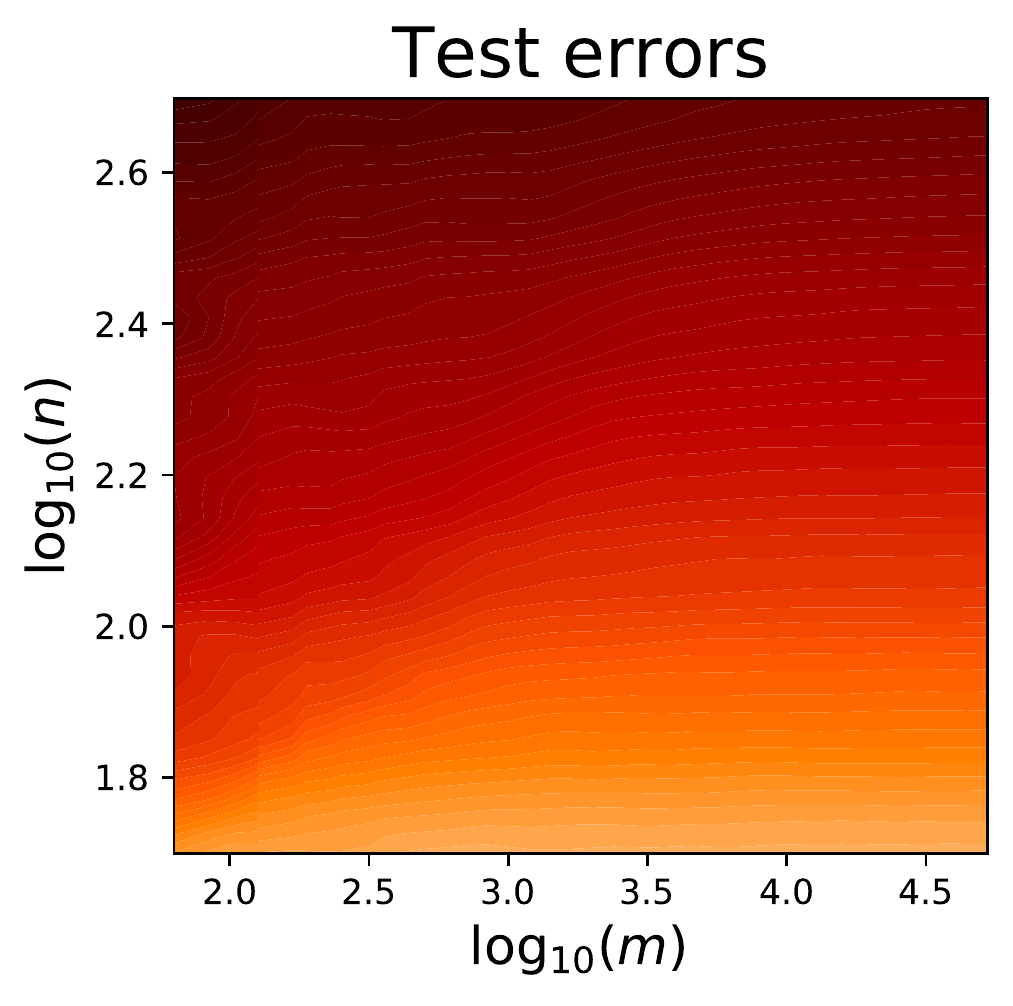}
     \vspace{-2mm}
    \caption{ \small How the network width affects the test error  of GD solutions. { The test errors are given in logarithmic scale.} These experiments are conducted on the single-neuron target function with $d=20$ and learning rate $\eta=0.0005$. { The GD is stopped when the training loss is smaller than $10^{-7}$}. The two dashed lines correspond to $m=n/(d+1)$ (left) and $m=n$ (right), respectively. \textbf{Left:} Conventional scaling; \textbf{Right:} Mean-field scaling.}
    \label{fig: heatmap}
\end{figure}

This can also be seen from the $m-n$ hyper-parameter space.  Shown in Figure \ref{fig: heatmap} are the heat maps of the
test errors under the conventional and mean-field scaling, respectively. We see that  the test error does not change much as $m$ changes for the mean-field scaling. In contrast, there is
a clear ``phase transition'' in the heat map for the conventional scaling when the training dynamics
undergoes a change from the neural network-like to a random feature-like behavior. This is further demonstrated in Figure \ref{fig: one-neuron-n200}: One can see that the performance of the neural network models indeed becomes very close to that of the random feature model as  the network width $m$ increases.
At the same time, the path norm also undergoes a sudden increase from the mildly over-parameterized/under-parameterized regime ($m\approx n/(d+1)$) to the highly over-parameterized regime ($m\approx n$). This may provide an explanation for the increase of the test error. 

\begin{figure}[!ht]
\centering
    \includegraphics[width=0.4\textwidth]{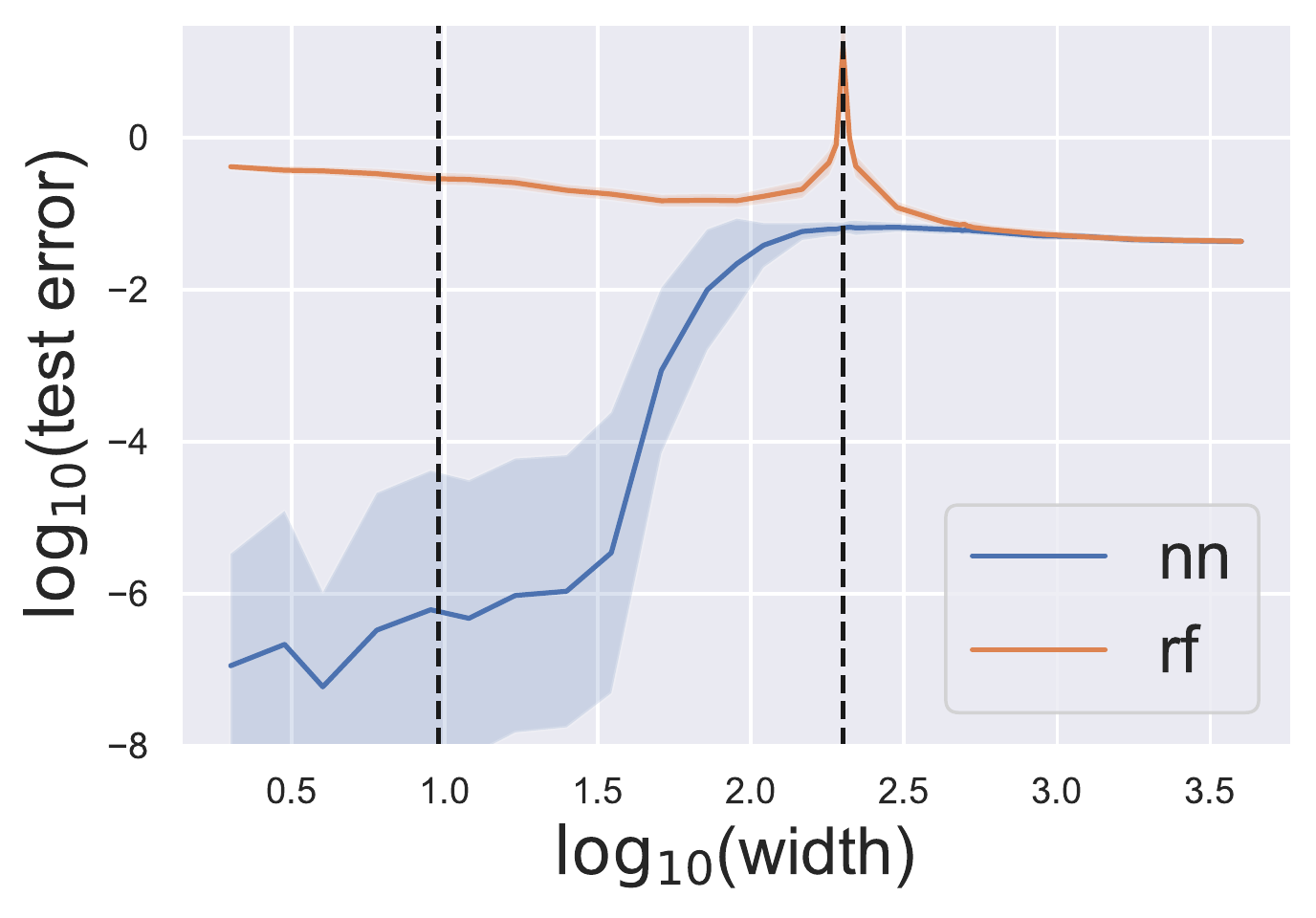}
    \includegraphics[width=0.4\textwidth]{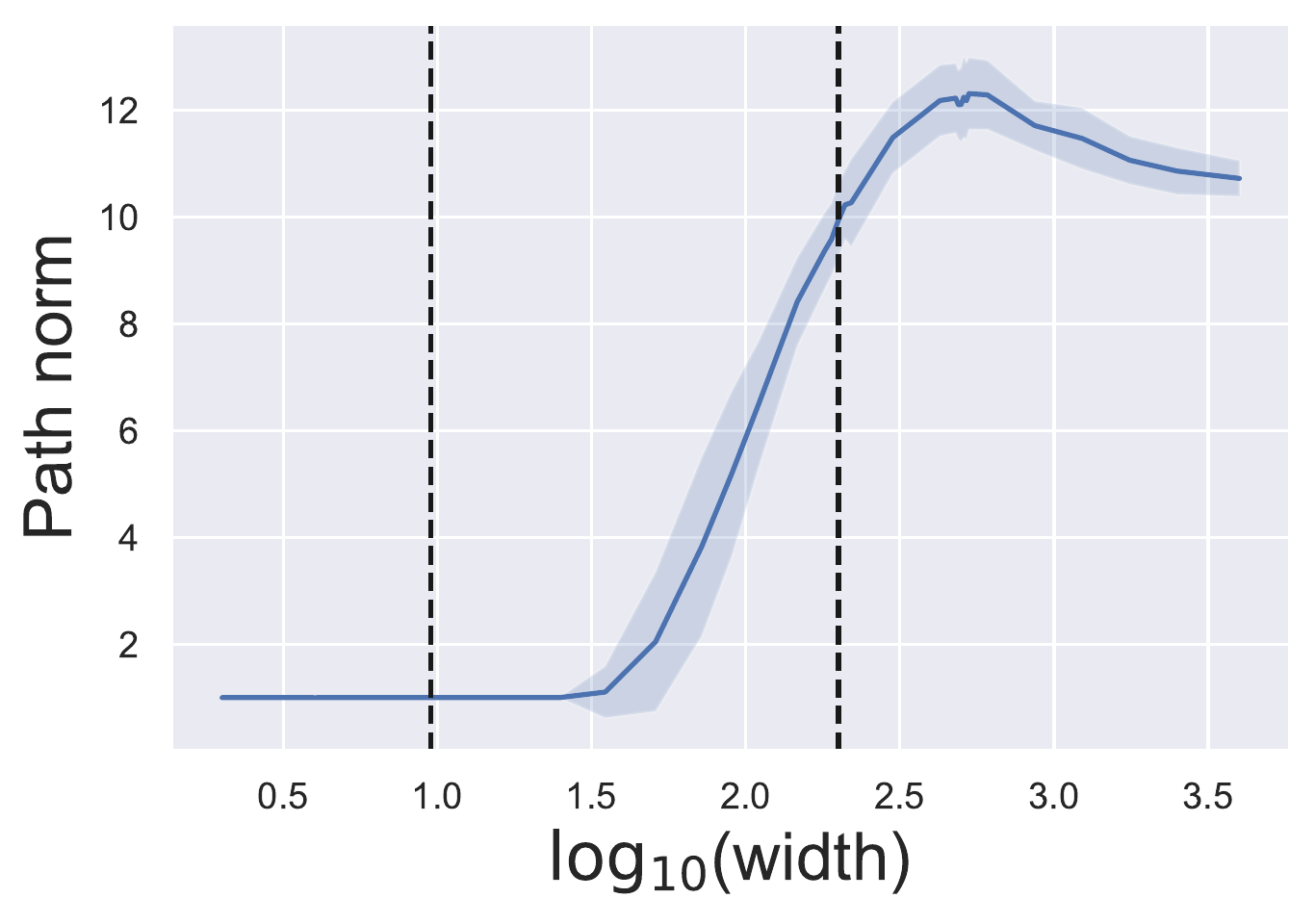}
    \caption{ \label{heatmap}\small Here $n=200, d=20$, learning rate = $0.001$. GD is stopped when the training error is smaller than $10^{-8}$. The two dashed  lines correspond to $m=n/(d+1)$ (left) and $m=n$ (right), respectively.  Several independent experiments were performed.
    The mean (shown as the line) and variance (shown in the shaded region) are shown here.
     \textbf{Left:} test error; \textbf{Right}: path norm.
    }
    \label{fig: one-neuron-n200}
\end{figure}

The existence of different kinds of behavior with drastically different generalization properties is one of the reasons
behind the fragility of deep learning: If the network architecture happens to fall into the random feature-like regime, the
performance will deteriorate.

\subsection{Other convergence results for the training of neural network models}
{
Convergence of GD for linear neural networks was proved in \cite{bartlett2018gradient,arora2018convergence,wu2019global}.
\cite{li2018algorithmic} analyzes the training of neural networks with quadratic activation function.
It is proved that a modified GD with small initialization converges to the ground truth in polynomial time if the sample size $n\geq O(dr^5)$ where $d,r$ are the input dimension and the rank of the target weight matrix, respectively. }
 \cite{soltanolkotabi2017learning} considers the learning of single neuron with SGD.  It was proved that as long as the sample size is large enough, SGD converges to the ground truth exponentially fast if the model also consists of a single neuron. Similar analysis for the population risk was presented in \cite{tian2017analytical}.

\subsection{Double descent and slow deterioration for the random feature model}

At the continuous (in time) level,  the training dynamics for the random feature model is a  linear system of ODEs
defined by the Gram matrix.  It turns out that the generalization properties for this linear model is
surprisingly complex, as we now discuss.

Consider a random feature model with features $\{\phi(\cdot; \bw)\}$ and probability distribution $\pi$ over the feature vectors $\bw$.
Let $\{\bw_1,\bw_2,...,\bw_m\}$ be the random feature vectors sampled from $\pi$. Let $\Phi$ be an $n\times m$ matrix with $\Phi_{ij}=\phi(\bx_i;\bw_j)$ where $\{\bx_1,\bx_2,...,\bx_n\}$ is the training data, and let
\begin{equation}
\hat{f}(\bx; \ba) = \sum\limits_{k=1}^m a_k\phi(\bx;\bb_k),
\end{equation}
where $\ba=(a_1,b_2,...,a_m)^T$are the parameters. To find $\ba$,  GD is used to optimize the following least squares objective function,
\begin{equation}\label{eq:obj_fcn}
\min_{\ba\in\bR^m} \frac{1}{2n}\left\| \Phi\ba-\by \right\|^2,
\end{equation}
starting from the origin, where $\by$ is a vector containing the values of the target function at $\{\bx_i,\ i=1,2,...,n \}$. 
The dynamics of $\ba$ is then given by
\begin{equation}
    \frac{d}{dt}\ba(t) = -\frac{1}{m}\frac{\partial}{\partial\ba}\frac{1}{2n}\left\| \Phi\ba-\by \right\|^2 = -\frac{1}{mn}\Phi^T(\Phi\ba-\by).
\end{equation}

Let $\Phi=U\Sigma V^T$ be the singular value decomposition of $\Phi$, where 
$$\Sigma=\mbox{diag}\{\lambda_1,\lambda_2,...,\lambda_{\min\{n,m\}}\}$$
with $\{\lambda_i\}$ being the singular values of $\Phi$, in descending order. Then the GD solution of~\eqref{eq:obj_fcn} at time $t\geq0$ is
given by
\begin{equation}\label{eq:solu}
\ba(t) = \sum\limits_{i:\lambda_i>0} \frac{1-e^{-\lambda_i^2t/(mn)}}{\lambda_i}(\bu_i^T\by)\bv_i.
\end{equation}
With this solution, we can conveniently compute the training and test error at any time $t$. Specifically, let $B=[\bw_1,...,\bw_m]$, and with an abuse of notation let $\phi(\bx;B)=(\phi(\bx;\bw_1),...,\phi(\bx;\bw_m))^T$,  the prediction function at time $t$ is given by
\begin{equation}\label{eq:pred_fcn}
\hat{f}_t(\bx) = \phi(\bx;B)^T\ba(t) = \sum\limits_{i:\lambda_i>0} \frac{1-e^{-\lambda_i^2t/mn}}{\lambda_i}(\bu_i^T\by)(\bv_i^T\phi(\bx;B)).
\end{equation}

Shown in the left figure of Figure \ref{fig: mnist2} is the test error of the minimum norm solution (which is the limit of the GD path when initialized 
at 0) of the random feature model as a function of the size of the model $m$ for $n=500$ for the 
MNIST dataset \cite{ma2019gd}. 
Also shown is the smallest eigenvalue of the Gram matrix.
One can see that the test error peaks at $m=n$ where the smallest eigenvalue of the Gram matrix becomes exceedingly small.
This is the same as the ``double descent''  phenomenon reported in \cite{belkin2019reconciling, advani2017high}.
But as can be seen from the figure (and discussed in more detail below), it is really a resonance kind of phenomenon caused by the appearance of very small eigenvalues of the Gram matrix when $m=n$.

Shown in right is the test error of the solution when the gradient descent training is stopped after different number of steps.
One curious thing is that when training is stopped at moderately large values of  training steps, the resonance behavior does not show up much.
Only when training is continued to very large number of steps does  resonance show up.

\begin{figure}[!h]
\centering
\includegraphics[width=0.43\textwidth]{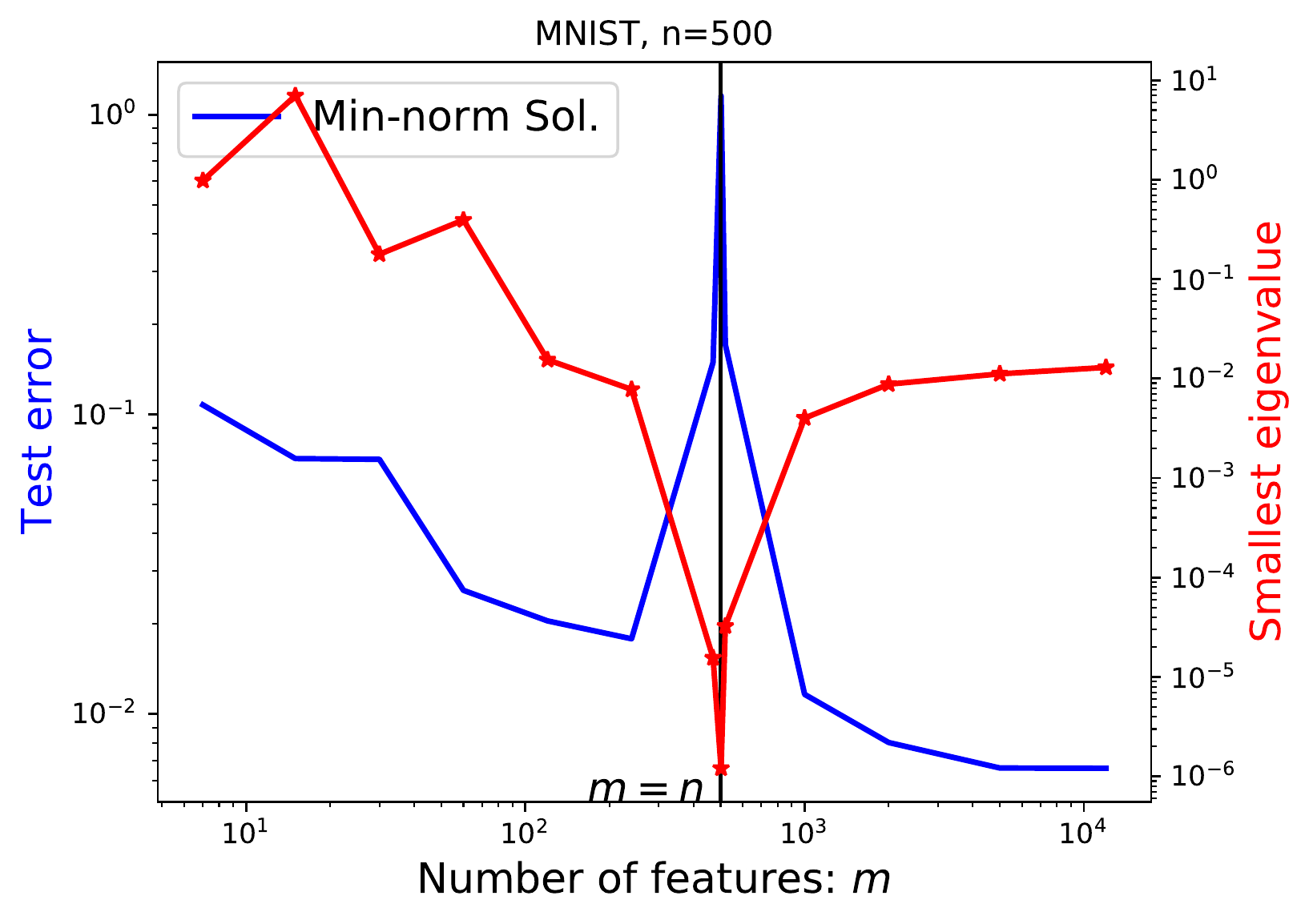}
\includegraphics[width=0.4\textwidth]{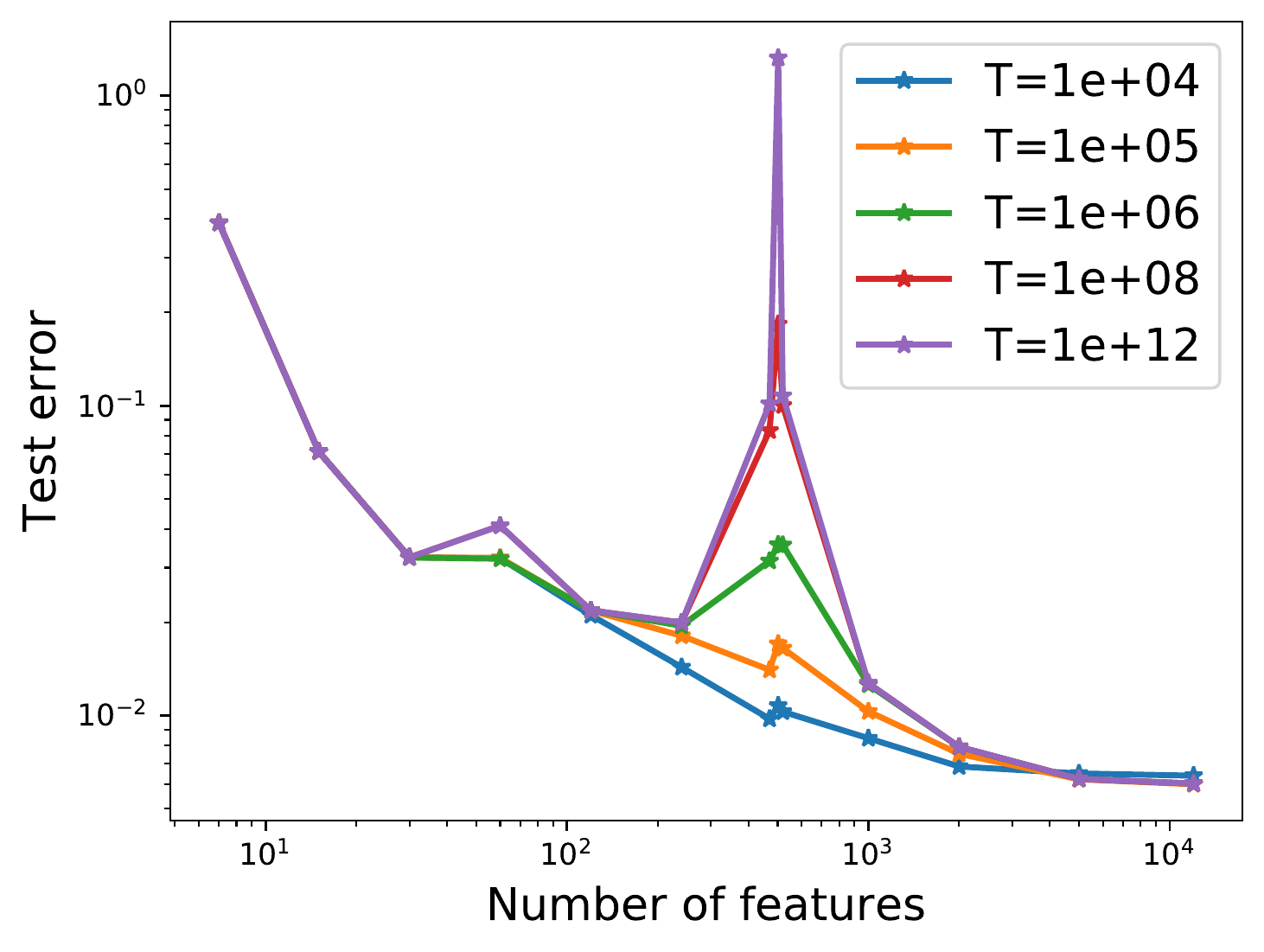}
\vspace{-2mm}
\caption{\small {\bf Left:} The test error and minimal eigenvalue of Gram matrix as a function of $m$  for $n=500$. 
 {\bf Right:} The test error of GD solutions obtained by running different number of iterations. MNIST dataset is used.
}
\label{fig: mnist2}
\end{figure}

Intuitively this is easy to understand. 
The large test error is caused by the small eigenvalues, { because each term in~\eqref{eq:pred_fcn} convergences to $1/\lambda_i$, and this limit is large when the corresponding eigenvalue is small.
However, still by~\eqref{eq:pred_fcn}, the contribution of the eigenvalues enter through terms like $e^{-\lambda^2 t}$, and therefore shows up very slowly in time.}

Some rigorous analysis of this can be found in \cite{ma2019gd} and will not be repeated here.
Instead we show in  Figure \ref{fig: slow-d} an example of the detailed dynamics of the training process.  One can see that the training dynamics
can be divided into three regimes.  In the first regime, the test error decreases from an $O(1)$ initial value to something small.
In the second regime, which spans several decades, the test error remains small.  It eventually becomes big
in the third regime when the small eigenvalues start to contribute significantly. { This slow deterioration phenomenon may also happen for more complicated models, like deep neural networks, considering that linear approximation of the loss function is effective near the global minimum, and hence in this regime the GD dynamics is similar to that of a linear model (e.g. random feature). Though, the conjecture needs to be confirmed by further theoretical and numerical study.}


\begin{figure}[H]
\caption{}
\label{fig: slow-d}
\centering
\includegraphics[width=\textwidth]{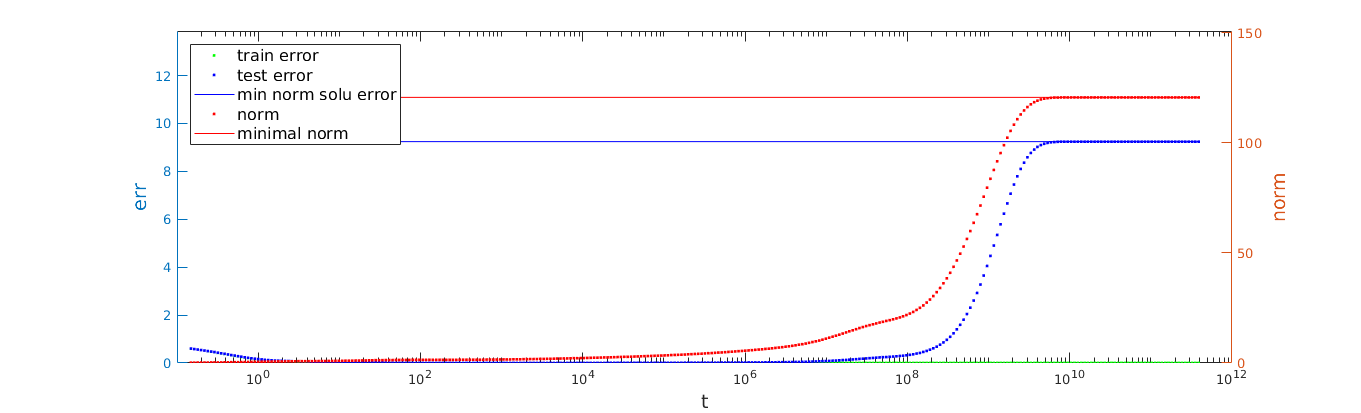}
\end{figure}
\vspace{-15mm}
\begin{equation*}
\color{red}
\hspace{5mm}\underbrace{\qquad\quad}_{I}
\underbrace{\qquad\qquad\qquad\qquad\qquad\qquad\qquad\qquad}_{II}
\underbrace{\qquad\qquad\qquad\qquad\qquad\qquad}_{III}
\end{equation*}

One important consequence of this resonance phenomenon is that it affects the training of two-layer neural networks under the conventional scaling. For example in Figure~\ref{fig: heatmap}, the test error of neural networks under the conventional scaling peaks around $m=n$, with $m$ being the width of the network. Though at $m=n$ the neural network has more parameters than $n$, it still suffers from 
this resonance phenomenon since as we saw before, in the early phase of the training,
the GD path for the  neural network model  follows closely that of the corresponding random feature model.

\subsection{Global minima selection}

In the over-parametrized regime, there usually exist many global minima. 
Different optimization algorithms may select different global minima.   For example,
it has been widely observed that SGD tends to pick ``flatter solutions'' than GD \cite{hochreiter1997flat,keskar2016large}.
A natural question is which global minima are selected by a particular optimization algorithm.

An interesting illustration of this is shown in Figure \ref{fig: escape}.  Here GD was used to train the FashionMNIST 
dataset to near completion,  and was suddenly replaced by SGD.
Instead of finishing the last step in the previous training process, the subsequent SGD path escapes from the vicinity of the
minimum that  GD was converging to, and
eventually converges to a different global minimum, with a slightly smaller test error \cite{keskar2016large}.

\begin{figure}[!h]
\centering
\includegraphics[width=0.4\textwidth]{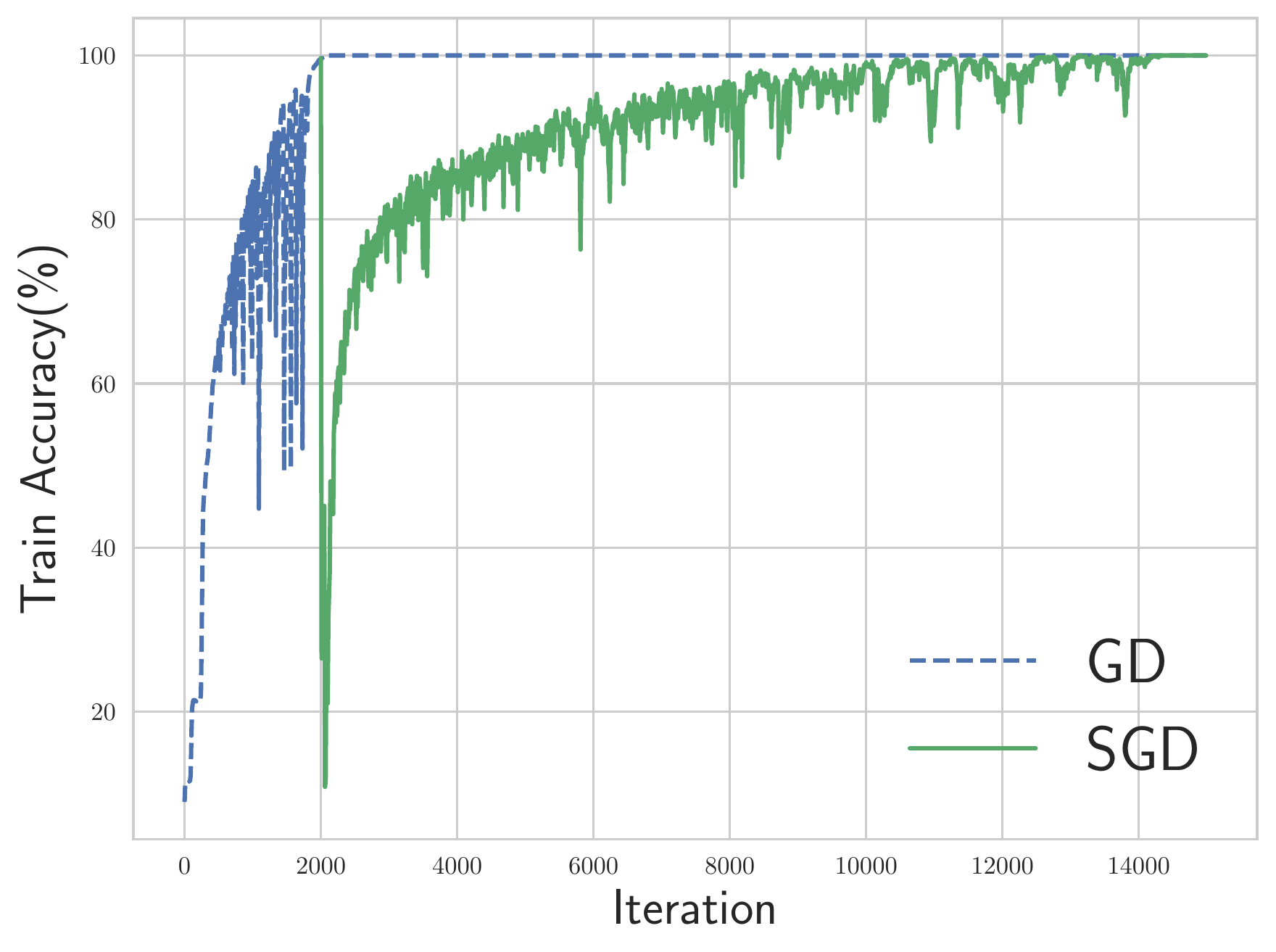} 
\includegraphics[width=0.4\textwidth]{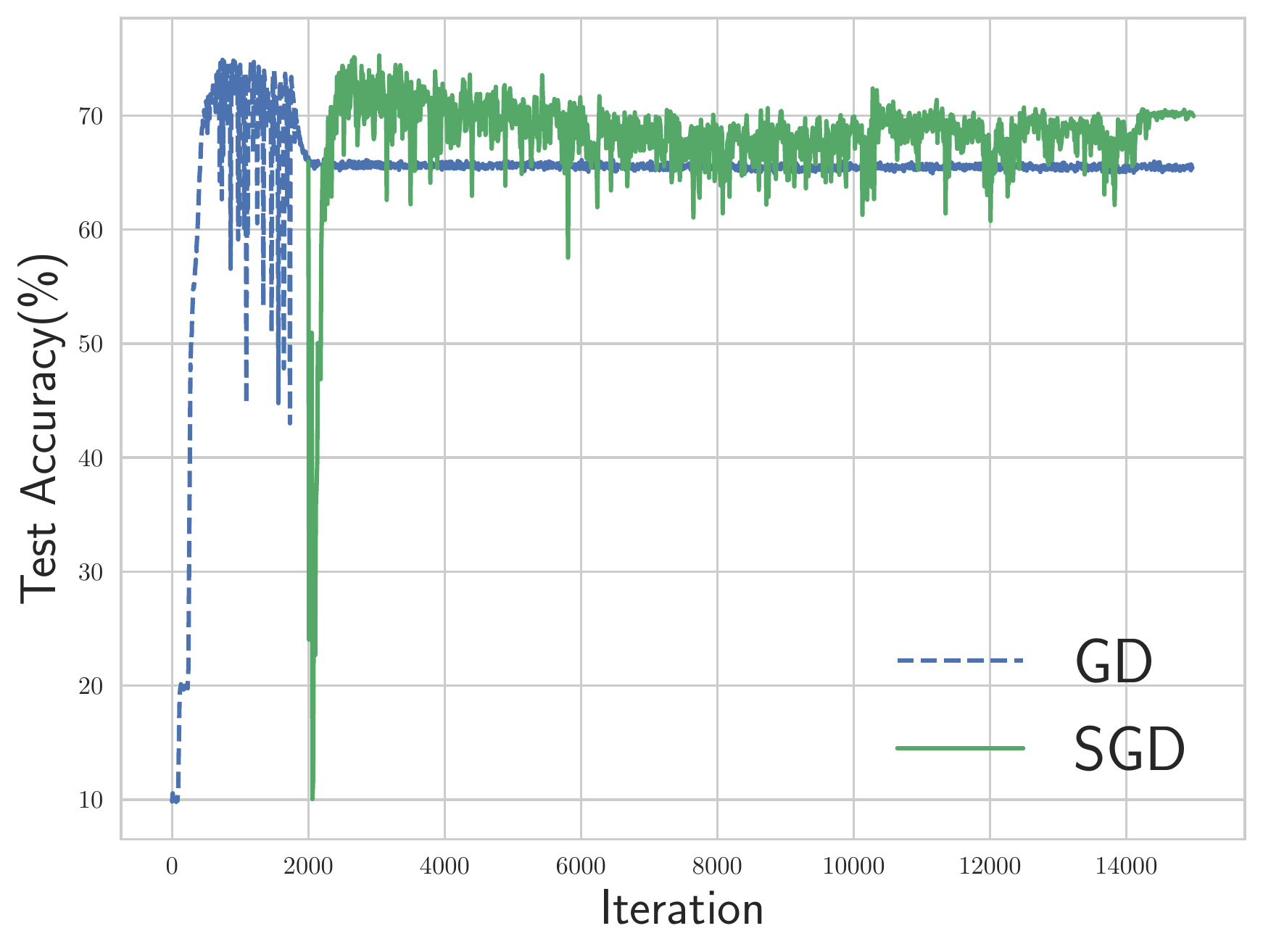}
\vspace{-2mm}
\caption{\small  Escape phenomenon in fitting corrupted FashionMNIST. The GD solution is unstable for the SGD dynamics. Hence the path escapes after GD is suddenly replaced by SGD. \textbf{Left:} Training accuracy, \textbf{Right:} Test accuracy. }
\label{fig: escape}
\end{figure}

This phenomenon can be well-explained by considering the dynamic stability of the optimizers, as was done in \cite{wu2018sgd}.
It was found that the set of dynamically stable global minima is different for different training algorithm.
In the example above, the global minimum that  GD was converging to was unstable for SGD.


The gist of this phenomenon can be understood from the 
 following simple one-dimensional optimization problem,
$$
f(x)=\frac{1}{2n}\sum\limits_{i=1}^{n}a_i x^2
$$
with $a_i\geq 0 \,\, \forall i\in [n]$.   The minimum is at $x=0$.
For GD with learning rate $\eta$ to be stable, the following has to hold:
$$ |1 - \eta a | \leq 1$$
SGD  is given by:
\begin{equation}\label{eq:sgd}
x_{t+1}=x_t-\eta a_{\xi} x_t = (1-\eta a_{\xi})x_t,
\end{equation}
where $\xi$ is a random variable that satisfies $\mathbb{P}(\xi=i)=1/n$. Hence, we have
\begin{eqnarray}
\bE x_{t+1}   &=& (1-\eta a)^{t} x_0, \\
\bE x_{t+1}^2 &=& \left[(1-\eta a)^2+\eta^2 s^2 \right]^t x_0^2,
\end{eqnarray} 
where $a=\sum_{i=1}^{n}a_i/n$, $s=\sqrt{\sum_{i=1}^n a_i^2/n-a^2}$. Therefore, for SGD to be stable at $x=0$, we not only need $|1-\eta a|\leq 1$, but also $(1-\eta a)^2+\eta^2 s^2 \leq 1$. 
In particular, SGD can only converge with the additional requirement that $s\leq 2/\eta$. 

The quantity $a$ is called sharpness, and $s$ is called non-uniformity in \cite{wu2018sgd}.
The above simple argument suggests that the global minima selected by SGD tend to be more uniform than the ones selected by GD.

This sharpness-non-uniformity criterion can be extended to multi-dimension.
It turns out that this  theoretical prediction is  confirmed quite well by practical  results.
Figure \ref{fig: nonuniformity} shows the sharpness and non-uniformity results of  SGD solutions for a VGG-type network
for the FashionMNIST dataset. We see that the predicted upper bound for the non-uniformity is both valid and quite sharp.

\begin{figure}[!h]
\centering
\includegraphics[width=0.4\textwidth]{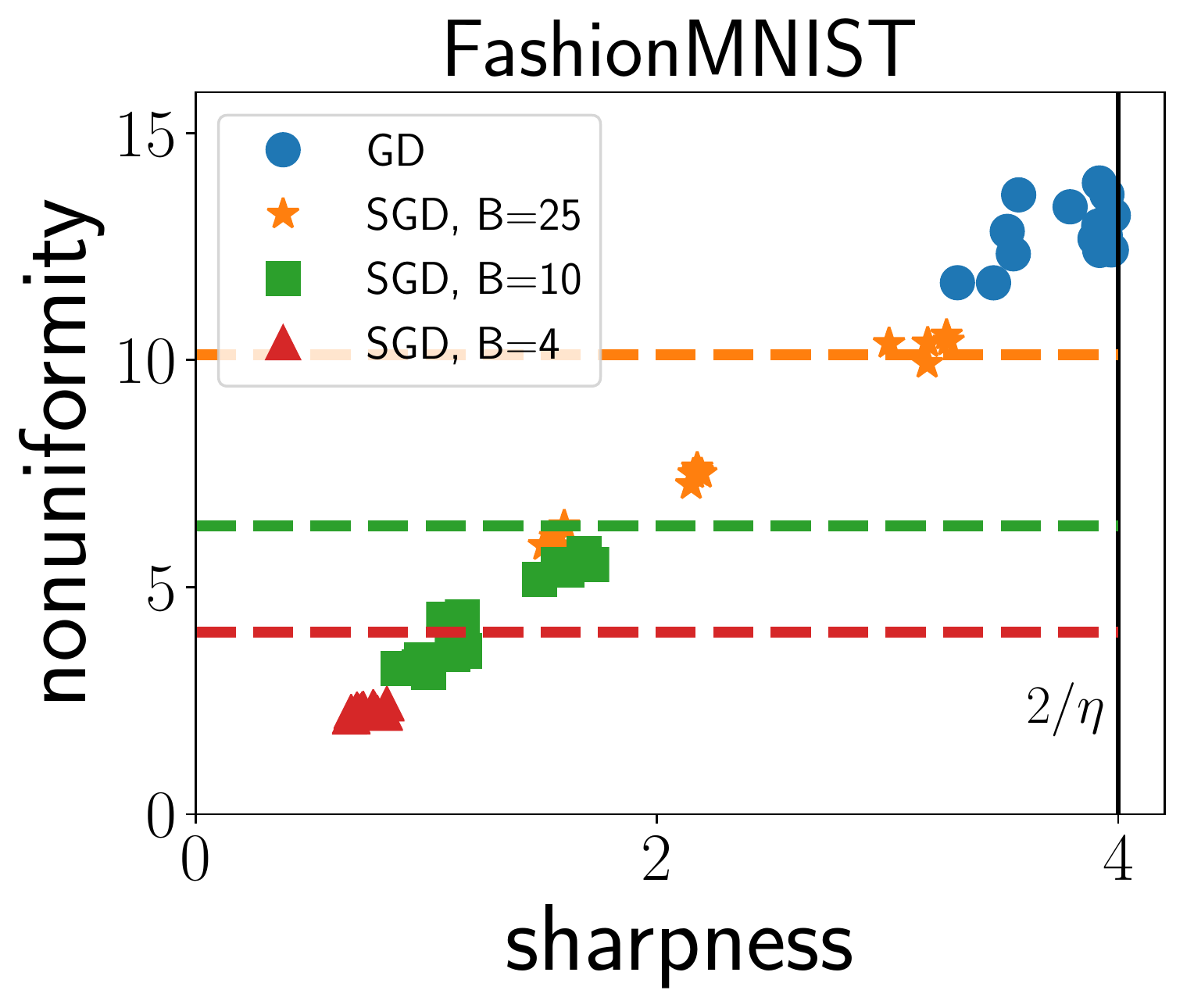}
\vspace{-2mm}
\caption{\small The sharpness-non-uniformity diagram for the minima selected by SGD applied to a VGG-type network for the
FashionMNIST dataset. Different colors correspond to different set of hyper-parameters. { B and $\mu$ in the figure stand for the batch size and the learning rate, respectively.} The dashed line shows the predicted bound for the non-uniformity. One can see that  (1) the data with different colors roughly lies below the corresponding dashed line and (2) the prediction given by the dashed line is quite
sharp.}
\label{fig: nonuniformity}
\end{figure}

\subsection{Qualitative properties of adaptive gradient algorithms}
Adaptive gradient algorithms are a family of optimization algorithms widely used for training neural network models.
These algorithms use a coordinate-wise scaling of the update direction (gradient or gradient with momentum) according to the history of the gradients. Two most popular adaptive gradient algorithms are RMSprop and Adam, whose update rules are 
\begin{itemize}
\item {\bf RMSprop:}
\begin{align}
  \bv_{t+1} &= \alpha \bv_{t} + (1-\alpha)(\nabla f(\bx_t))^2 \nonumber\\
  \bx_{t+1} &= \bx_t - \eta\frac{\nabla f(\bx_t)}{\sqrt{\bv_{t+1}}+\epsilon} \label{eqn: rmsprop}
\end{align}
\item {\bf Adam:}
\begin{align}
  \bv_{t+1} &= \alpha \bv_{t} + (1-\alpha) (\nabla f(\bx_t))^2 \nonumber\\
  \bfm_{t+1} &= \beta \bfm_{t} + (1-\beta) \nabla f(\bx_t) \nonumber\\
  \bx_{t+1} &= \bx_t - \eta\frac{\bfm_{t+1}/(1-\beta^{t+1})}{\sqrt{\bv_{t+1}/(1-\alpha^{t+1})}+\epsilon} \label{eqn: adam}
\end{align}
\end{itemize}
In~\eqref{eqn: rmsprop} and~\eqref{eqn: adam}, $\epsilon$ is a small constant used to avoid division by $0$, usually taken to be $10^{-8}$.
It is added to each component of the vector in the denominators of these equations.
The division should also be understood as being component-wise. 
Here we will focus on Adam. For further discussion, we refer to~\cite{ma2020quali}.

One important tool for understanding these adaptive gradient algorithms is their continuous limit.
Different continuous limits can be obtained from different ways of taking limits. 
If we let $\eta$ tends to $0$ while keeping $\alpha$ and $\beta$ fixed, the limiting dynamics will be 
\begin{equation}\label{eqn: ode1}
    \dot{\bx} = -\frac{\nabla f(\bx)}{|\nabla f(\bx)|+\epsilon}.
\end{equation}
This becomes signGD when $\epsilon=0$. On the other hand, if we let $\alpha=1-a\eta$ and $\beta=1-b\eta$ and take $\eta\rightarrow0$, then we
obtain the limiting dynamics
\begin{align}
  \dot{\bv} &= a(\nabla f(\bx)^2-\bv) \nonumber\\
  \dot{\bfm} &= b(\nabla f(\bx)-\bfm) \nonumber\\
  \dot{\bx} &= -\frac{(1-e^{-bt})^{-1}\bfm}{\sqrt{(1-e^{-at})^{-1}\bv}+\epsilon} \label{eqn: adam_ode}
\end{align}

In practice the loss curves of Adam can be very complicated. The left panel of Figure~\ref{fig: adam_heatmap_nn} shows an example. Three obvious features can be observed from these curves: 
\begin{enumerate}
    \item {\bf Fast initial convergence:} the loss curve decreases very fast,  sometimes even super-linearly, at the early stage of the training.
    \item {\bf Small oscillations:} The fast initial convergence is followed by oscillations around the minimum. 
    \item {\bf Large spikes:} spikes are sudden increase of the value of the loss. They are followed by an oscillating recovery. Different from small oscillations, spikes make the loss much larger and the interval between two spikes is longer. 
\end{enumerate}


The fast initial convergence can be partly explained by the convergence property of signGD, which  attains global minimum in finite time 
for strongly convex objective functions. Specifically, we have the following proposition \cite{ma2020quali}.
\begin{proposition}
Assume that the objective function satisfies the Polyak-Lojasiewicz (PL) condition: $\|\nabla f(\bx)\|_2^2 \geq \mu f(\bx)$, for some positive constant
$\mu$.
Define continuous signGD dynamics,
\begin{equation*}
    \dot{\bx}_t = -\textrm{sign}(\nabla f(\bx_t)),
\end{equation*}
then we have 
\begin{equation*}
    f(\bx_t) \leq \left(\sqrt{f(\bx_0)} - \frac{\sqrt{\mu}}{2}t\right)^2.
\end{equation*}
\end{proposition}

The small oscillations and spikes  may potentially be explained by linearization around the stationary point, though in this case, linearization is quite tricky due to the singular nature of the stationary point~\cite{ma2020quali}.

The performance of Adam depends sensitively on the  values of $\alpha$ and $\beta$. This has also been studied in~\cite{ma2020quali}. 
Recall that $\alpha=1-a\eta$ and $\beta=1-b\eta$. Focusing on the region where $a$ and $b$ are not too large,  three regimes with different behavior patterns were observed in the hyper-parameter space of $(a,b)$:
\begin{enumerate}
    \item {\bf The spike regime}:  This happens when $b$ is sufficiently larger than $a$. In this regime large spikes appear in the loss curve, which makes the optimization process unstable.
    
    \item {\bf The oscillation regime}: This happens when $a$ and $b$ have similar magnitude (or in the same order). In this regime the loss curve exhibits fast and small oscillations. Small loss and stable loss curve can be achieved.
    
    \item {\bf The divergence regime}: This happens when $a$ is sufficiently larger than $b$. In this regime the loss curve is unstable and usually diverges after some period of training. This regime should be avoided in practice since the training loss stays large. 
\end{enumerate}
In Figure~\ref{fig: adam_patterns} we show one typical loss curve for each regime for a typical neural network model. 
\begin{figure}[!h]
    \centering
    \includegraphics[width=0.32\textwidth]{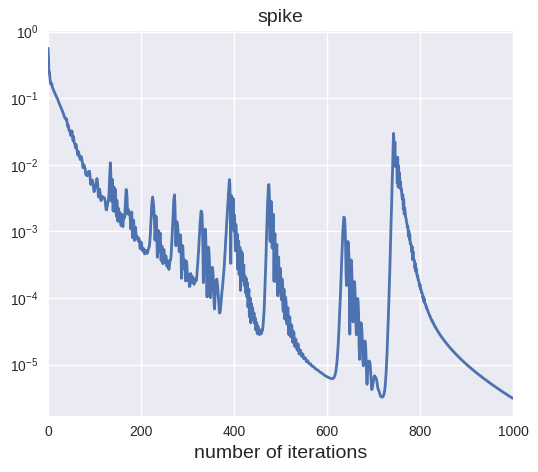}
    \includegraphics[width=0.32\textwidth]{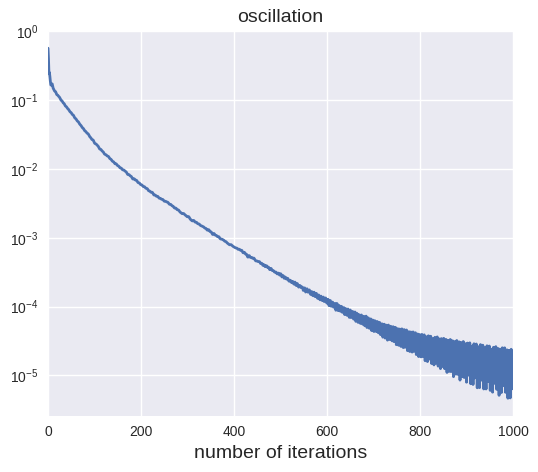}
    \includegraphics[width=0.32\textwidth]{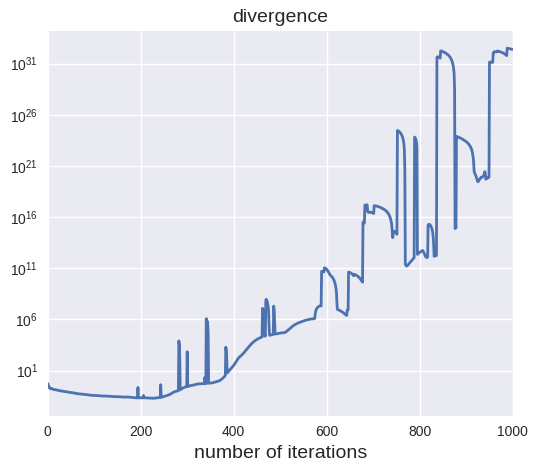}
    \includegraphics[width=0.32\textwidth]{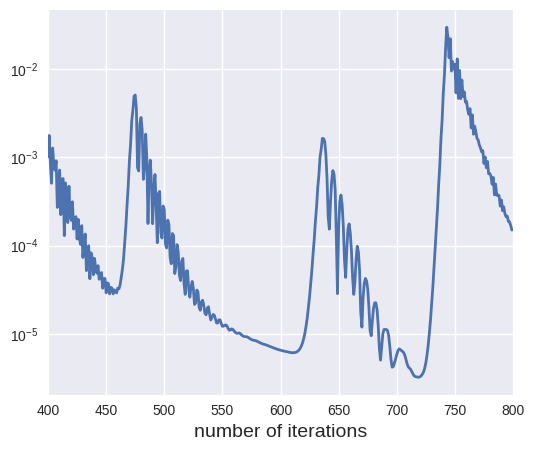}
    \includegraphics[width=0.32\textwidth]{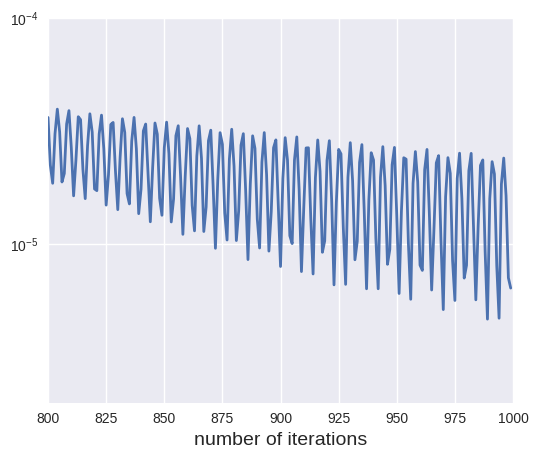}
    \includegraphics[width=0.32\textwidth]{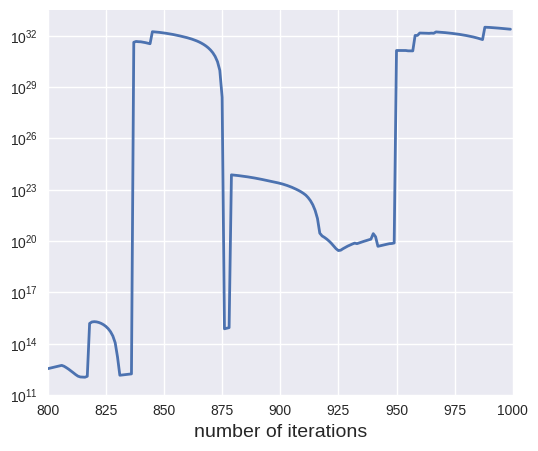}
    \vspace{-2mm}
    \caption{\small The three typical behavior patterns for Adam  trajectories.  Experiments are conducted on a fully-connected neural network with three hidden layers are $256,128,64$, respectively. The training data is taken from 2 classes of CIFAR-10 with 1000 samples per class. The learning rate is $\eta=0.001$. The first row shows the loss curve of total $1000$ iterations, the second row shows part of the loss curve (the last $200$ iterations for oscillation and divergence regimes, and $400-800$ iterations for the spike regime). {\bf Left:} $a=1$, $b=100$, large spikes appear in the loss curve; {\bf Middle:} $a=10$, $b=10$, the loss is small and oscillates very fast, and the amplitude of the oscillation is also small; {\bf Right:} $a=100$, $b=1$, the loss is large and blows up.}
    \label{fig: adam_patterns}
\end{figure}


In Figure~\ref{fig: adam_heatmap_nn}, we study Adam on a neural network model and show the final average training loss for different 
values of  $a$ and $b$. One can see that Adam can achieve very small loss in the oscillation regime. The algorithm is not stable in the spike regime and may blow up in the divergence regime. These observations suggest that in practice one should take $a\approx b$ with small values of $a$ and $b$. Experiments on more complicated models, such as ResNet18, also support these conclusions~\cite{ma2020quali}.
 
\begin{figure}[!h]
\centering
\includegraphics[width=0.29\textwidth]{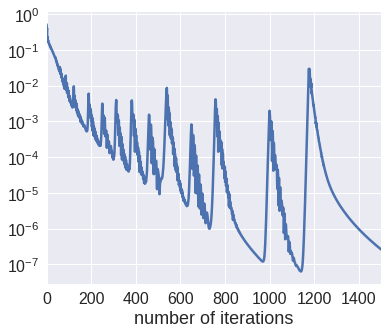}
\hspace{-5mm}
\includegraphics[width=0.37\textwidth, height=0.26\textwidth]{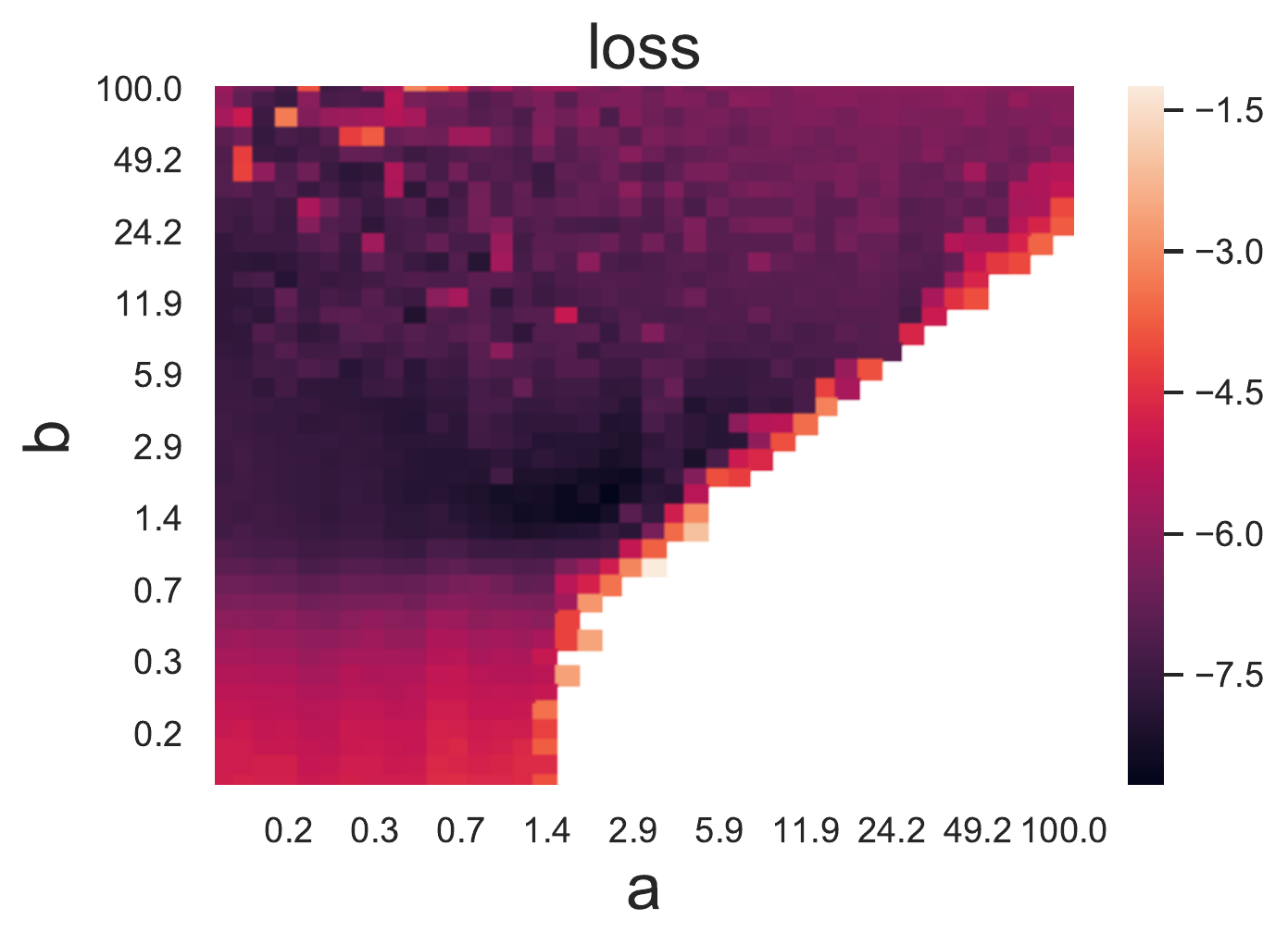}
\hspace{-5mm}
\includegraphics[width=0.33\textwidth]{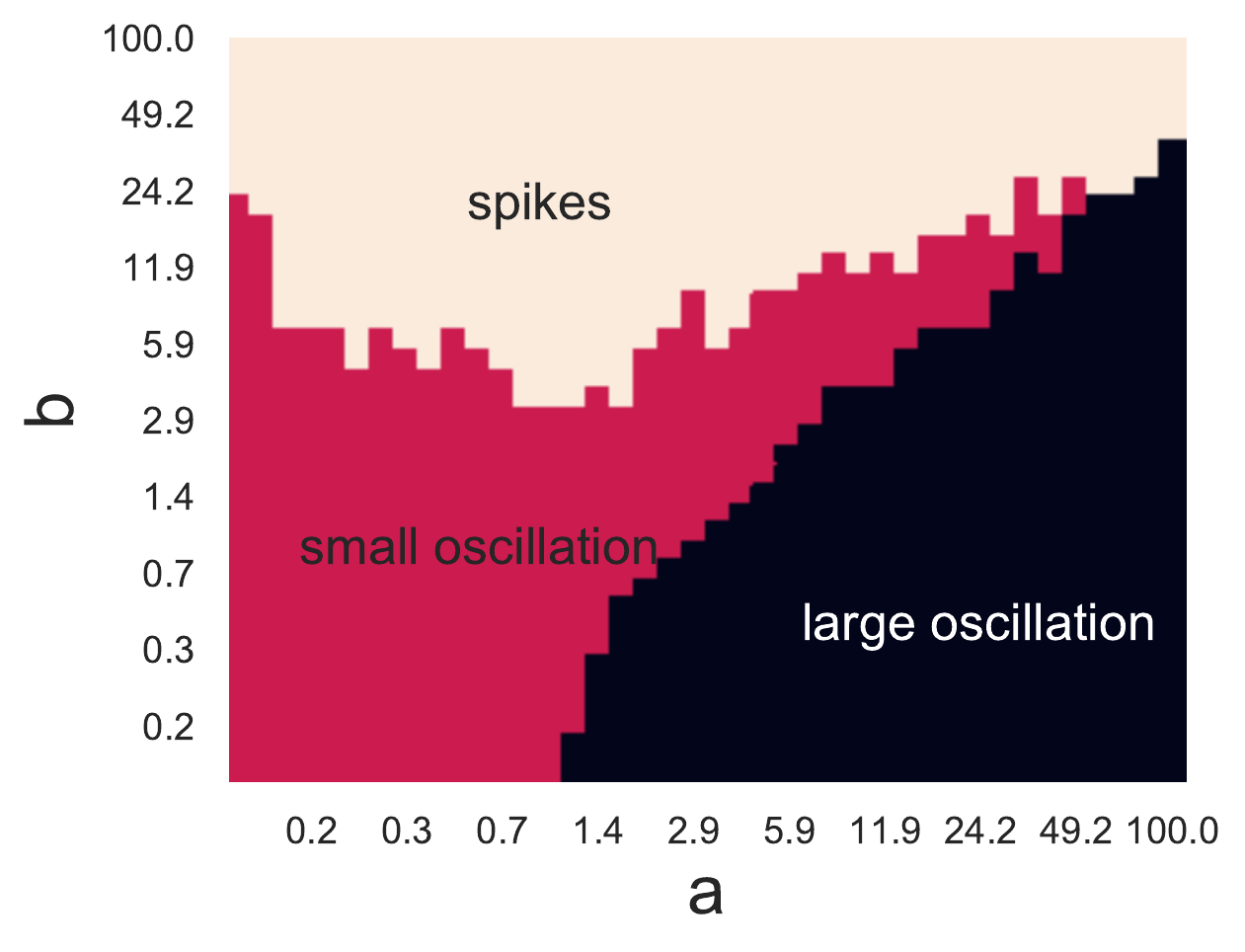}
\vspace{-2mm}
\caption{\small {\bf Left:} The training curve of  Adam  for a multi-layer neural network model on CIFAR-10. The network has $3$ hidden-layers whose widths are 256-256-128, with learning rate $1e-3$ and $(\alpha, \beta)=(0.9, 0.999)$ as default. $2$ classes are picked from CIFAR-10 with $500$ images in each class. Square loss function is used. {\bf Middle:} Heat map of the average training loss (in logarithmic scale) of Adam on a multi-layer neural network model. The loss is the averaged over
the last $1000$ iterations. $a$ and $b$ range from $0.1$ to $100$ in logarithm scale. {\bf Right:} The different  behavior patterns.}
\label{fig: adam_heatmap_nn}
\end{figure}

\subsection{Exploding and vanishing gradients for multi-layer neural networks}

Exploding and vanishing gradients is one of the main obstacles for training of multi (many)-layer neural networks.
Intuitively it is easy to see why this might be an issue.
The gradient of the loss function with respect to the parameters involve a product of many matrices. 
Such a product can easily grow or diminish very fast as the number of products, namely the layers, increases.

Consider the multi-layer neural network with depth $L$:
\begin{align}
\bm{z}_0&=\bm{x},\nonumber\\
\bm{z}_l&=\sigma (W_l\bm{z}_{l-1}+\bm{c}_l),\quad l=1,2,\cdots, L,\label{FullyConnected}\\
f(\bm{x};\theta)&=\bm{z}_L,\nonumber\
\end{align}
where $\bm{x}\in\mathbb{R}^{d}$ is the input data, $\sigma $ is the ReLU activation and $\theta:=(W_1, \bm{c}_1, \cdots, W_d, \bm{c}_L)$ denotes the collection of parameters. Here $W_l\in\mathbb{R}^{m_l\times m_{l-1}}$ is the weight, $\bm{c}_l\in\mathbb{R}^{m_l}$ is the bias. $m_l$ is  the width of the $l$-th hidden layer. 

To make quantitative statements,  we consider the case when the weights are i.i.d. random variables.
This is typically the case when the neural networks are initialized. This problem has been studied in 
\cite{hanin2018neural, hanin2018start}. Below is a brief summary of the results obtained in these papers.

Fix two collections of probability measures $\mu=(\mu^{(1)},\mu^{(2)},\cdots,\mu^{(L)})$,
 $\nu=(\nu^{(1)},\nu^{(2)},\cdots,\nu^{(L)})$ on $\mathbb{R}$ such that
	\begin{itemize}
		\item $\mu^{(l)}$, $\nu^{(l)}$ are symmetric around 0 for every $1\le l\le L$;
		\item  {the variance of $\mu^{(l)}$ is $2/n_{l-1}$}; 
		\item $\nu^{(l)}$ has no atoms.
	\end{itemize}
	We consider the random network obtained by:
\begin{align}
W_l^{i,j}\sim \mu^{(l)}, \quad \bm{c}_l^{i}\sim \nu^{(l)},\qquad i.i.d.
\end{align}
for any $i=1,2,\cdots,m_l$ and $j=1,2,\cdots,m_{l-1}$,
i.e. the weights and biases  at layer $l$ are drawn independently from $\mu^{(l)}$, $\nu^{(l)}$  respectively. 
Let
\begin{align}\label{JN}
Z_{p,q}:=\frac{\partial \bm{z}_L^q}{\partial \bm{x}^p},\qquad p=1,2,\cdots, m_0=d, \quad q=1,2,\cdots, m_L. 
\end{align}

\begin{theorem}\label{thm:AnnealedEVGP}
We have
	\begin{align}
	\mathbb{E}\left[Z^2_{p,q}\right]=\frac{1}{d}.
	\end{align}
	In contrast, the fourth moment of $Z_{p,q}$ is \textbf{exponential} in $\sum_{l}\frac{1}{m_l}$:
	\begin{align}
	\frac{2}{d^2}\exp\left(\frac{1}{2}\sum_{l=1}^{L-1}\frac{1}{m_l}\right)\le\mathbb{E}\left[Z^4_{p,q}\right]\le \frac{C_{\mu}}{d^2}\exp\left(C_{\mu}\sum_{l=1}^{L-1}\frac{1}{m_l}\right),
	\end{align}
	where $C_{\mu}>0$ is a constant only related to the (fourth) moment of $\mu=\{\mu^{(l)}\}_{l=1}^L$.
	
	Furthermore, for any $K\in\mathbb{Z}_+$, $3\le K< \min\limits_{1\le l\le L-1}\{m_l\}$, we have
	\begin{align}
	\mathbb{E}\left[Z^{2K}_{p,q}\right]\le \frac{C_{\mu,K}}{d^K}\exp\left(C_{\mu,K}\sum_{l=1}^{L-1}\frac{1}{m_l}\right),
	\end{align}
	where $C_{\mu,K}>0$ is a constant depending on $K$ and the (first $2K$) moments of $\mu$.
\end{theorem}

Obviously, by Theorem $\ref{thm:AnnealedEVGP}$, to avoid the exploding and vanishing gradient problem, we want the quantity
$\sum_{l=1}^{L-1}\frac{1}{m_l} $ to be small.  This is also borne out from numerical experiments \cite{hanin2018start, hanin2018neural}.

 \subsection{What's not known?}
 
 There are a lot that we don't know about training dynamics.  Perhaps the most elegant mathematical question is the 
 global convergence of the mean-field equation for two-layer neural networks.
 \[
 \partial_t \rho = \nabla (\rho \nabla V), \quad V= \frac{\delta \mathcal{R}}{\delta \rho}
 \]
 The conjecture is that:
 \bi
 \item  if $\rho_0$ is a smooth distribution with full support, { then} the dynamics described by this flow should converge and
 the population risk $\mathcal{R}$ should converge to 0;

 \item if the target function $f^*$ lies in the Barron space, then the Barron norm of the output function stays uniformly bounded.
 \ei
 Similar statements should also hold for the empirical risk. 
 
 { 
 Another core question is when two-layer neural networks can be trained efficiently~\footnote{The authors would like to thank Jason Lee for helpful conversations on the topic.}. The work \cite{Livni,shamir2018distribution} shows that there exist target functions with Barron norms of size  $poly(d)$, such that the training is exponentially slow in the statistical-query (SQ) setting \cite{kearns1998efficient}. A concrete example is $f^*(\bx)=\sin(dx_1)$ with $\bx\sim\mathcal{N}(0,I_{d})$, for which it is easy to verify that the Barron norm of $f^*$ is $poly(d)$. However, 
 \cite{shamir2018distribution} shows that the gradients of the corresponding neural networks are exponentially small, i.e. $O(e^{-d})$. Thus even for a moderate $d$,  it is impossible to evaluate the gradients accurately on a finite-precision machine due to the floating-point error. Hence gradient-based optimizers are unlikely to succeed.  This suggests that the Barron space is very likely too large for studying the  training of two-layer neural networks. It is an open problem to identify the right function space, such that the functions can be learned in polynomial time by two-layer neural networks. 
 }

 
 {
 There are (at least) two possibilities for why the training becomes slow for certain (Barron) target functions in high dimension:
 \begin{enumerate}
     \item The training is slow in the continuous model due to the large parameter space or
     \item The training is fast in the continuous model with dimension-independent rates, but the discretization becomes more difficult in high dimension.
 \end{enumerate}
 It is unknown which of these explanations applies. Under strong conditions, it is known that parameters which are initialized according to a law $\pi_0$ and trained by gradient descent perform better at time $t>0$ than parameters which are drawn from the distribution $\pi_t$ given by the Wasserstein gradient flow starting at $\pi_0$ \cite{chen2020dynamical}. This might suggest that also gradient flows with continuous initial condition may not reduce risk at a dimension-independent rate.}
 
 
 One can ask similar questions for multi-layer neural networks and residual neural networks.
 However, it is much more natural to formulate them using the continuous formulation.  We will postpone this to a separate article.


Even less is known for the training dynamics of neural network models under the conventional scaling, 
except for the high over-parametrized regime. This issue is all the more important since conventional scaling is the
overwhelming scaling used in practice.

In particular, since the neural networks used in practice are often over-parametrized, and they seem to perform much better
than random feature models, some form of implicit regularization must be at work.  Identifying the presence and the mechanism of such implicit
regularization is a very important question for understanding the success of neural network models.

We have restricted our discussion to gradient descent training. What about stochastic gradient descent and other
training algorithms?

 \section{Concluding remarks}

Very briefly, let us summarize the main theoretical results that have been established so far.

\begin{enumerate}
\item  Approximation/generalization properties of hypothesis space:
\bi
\item Function spaces and quantitative measures for the  approximation properties for various machine learning models.
 The random feature model is naturally associated with the corresponding RKHS.
 In the same way,  the Barron norm is identified as the natural measure associated with two-layer neural network models
 (note that this is different from the spectral norm defined by Barron).
 For residual networks, the corresponding quantities are defined for the flow-induced spaces.
 For multi-layer networks, a promising candidate is provided by the multi-layer norm.
\item Generalization error estimates of regularized model.  Dimension-independent error rates have been established for these models.
Except for multi-layer neural networks, these error rates are comparable to Monte Carlo.
They provide a way to compare these different machine learning models and serve as a benchmark for studying
implicit regularization.
\ei

\item  Training dynamics for highly over-parametrized neural network models: 
\bi
\item Exponential convergence for the empirical risk.
\item Their generalization properties are no better than the corresponding random feature model or kernel method.
\ei

\item  Mean-field training dynamics for two-layer neural networks:
\bi
\item If the initial distribution has full support and the GD path converges, then it must converge to a global minimum.
\ei

\end{enumerate}

A lot has also been learned from careful numerical experiments and partial analytical arguments, such as:
\bi
\item Over-parametrized networks may be able to interpolate any training data.
\item The ``double descent'' and ``slow deterioration'' phenomenon for the random feature model and their effect on the
corresponding neural network model.
\item The qualitative behavior of adaptive optimization algorithms.
\item The global minima selection mechanism for different optimization algorithms.
\item The phase transition of the generalization properties of two-layer neural networks under the conventional scaling.
\ei

We have mentioned many open problems throughout this article. Besides the rigorous mathematical results that are called for,
we feel that carefully designed numerical experiments should also be encouraged. In particular, they might give some insight on
the difference between neural network models of different depth (for example, two-layer and three-layer neural neworks),
and the difference between scaled and unscaled residual network models.

One very important issue that we did not discuss much is the continuous formulation of machine learning.
We feel this issue deserves a separate article when the time is ripe.


{\bf Acknowledgement:}. This work is supported in part by a gift to the Princeton University from iFlytek.

{\small 
\bibliographystyle{plain}
\bibliography{ref}
}

\appendix

\section{Proofs for Section \ref{sec: rf}}
\paragraph{Proof of Theorem \ref{thm: rf-invers-approx} }
Write the random feature model as
\[
f_m(\bx) =\int a\phi(\bx;\bw) \rho_m(d a,d\bw),
\]
with
\[
    \rho_m(a,\bw) = \frac{1}{m}\sum_{j=1}^m \delta(a-a_j)\delta(\bw - \bw_j).
\]
Since $\supp(\rho_m)\subseteq K:= [-C,C]\times \Omega$, the sequence of probability measures  $(\rho_m)$ is tight. By Prokhorov's theorem, there exists a subsequence $(\rho_{m_k})$ and a probability measure $\rho^*\in \mathcal{P}(K)$ such that $\rho_{m_k}$ converges weakly to $\rho^*$. Due to that $g(a,\bw;\bx)=a\phi(\bx;\bw)$ is bounded and continuous with respect to $(a,\bw)$ for any $\bx\in [0,1]^d$,  we  have 
\begin{equation}\label{eqn: random-feature-re}
f^*(\bx)=\lim_{k\to\infty} \int a\phi(\bx;\bw)\rho_{m_k}(da, d\bw)  = \int a\phi(\bx;\bw) \rho^*(da,d\bw),
\end{equation}
Denote by $a^*(\bw)=\int a\rho^*(a|\bw)da$ the conditional expectation of $a$ given $\bw$. Then we have 
\[
    f^*(\bx) = \int a^*(\bw)\phi(\bx;\bw) \pi^*(d\bw),
\]
where $\pi^*(\bw)=\int \rho(a,\bw)da$ is the marginal distribution. Since $\supp(\rho^*)\subseteq K$, it follows that $|a^*(\bw)|\leq C$. For any bounded and continuous function $g: \Omega\mapsto \RR$, we have 
\begin{align}
    \int g(\bw) \pi^*(\bw)d\bw &= \int g(\bw)  \rho^*(a,\bw)d\bw da \\
    &= \lim_{k\to \infty} \int g(\bw) \rho_m^*(a,\bw)da d\bw \\
    &=\lim_{k\to\infty} \frac{1}{m}\sum_{j=1}^{m_k}g(\bw_j^0) = \int g(\bw) \pi_0(\bw)d\bw.
\end{align}
By the strong law of large numbers, the last equality holds with probability $1$. Taking $g(\bw)=a^*(\bw)\phi(\bx;\bw)$, we obtain 
\[
f^*(\bx) = \int a^*(\bw)\phi(\bx;\bw)\pi_0(\bw)d\bw.
\]

To prove Theorem \ref{thm: rf-a-priori}, we first need the following result (\cite[Proposition 7]{e2019continuous}).
\begin{proposition}\label{pro: approx-empirical-error}
Given training set $\{(\bx_i,f(\bx_i))\}_{i=1}^n$, for any $\delta\in (0,1)$, we have that with probability at least $1-\delta$ over the random sampling of $\{\bw_j^0\}_{j=1}^m$, there exists $\ba\in\RR^m$ such that 
\begin{equation}\label{eqn: emp-approx-rf}
\begin{aligned}
\hat{\CR}_n(\ba)&\leq \frac{2\log(2n/\delta)}{m} \|f\|^2_{\cH}  +  \frac{8\log^2(2n/\delta)}{9m^2} \|f\|^2_{\infty},\\
    \frac{\|\ba\|^2}{m} &\leq \|f\|_{\cH}^2 + \sqrt{\frac{\log(2/\delta)}{2m}} \|f\|^2_{\infty}
\end{aligned}
\end{equation}
\end{proposition}

\paragraph{Proof of Theorem \ref{thm: rf-a-priori}}
Denote by $\tilde{\ba}$ the solution constructed in Proposition \ref{pro: approx-empirical-error}. Using the definition, we have 
\[
\hat{\cR}_n(\hat{\ba}_{n}) + \frac{\|\hat{\ba}_n\|}{\sqrt{nm}} \leq \hat{\cR}_n(\tilde{\ba}) +  \frac{\|\tilde{\ba}\|}{\sqrt{nm}}.
\]
Hence, 
$
    \frac{\|\hat{\ba}_n\|}{\sqrt{m}} \leq C_{n,m}:= \frac{\|\tilde{\ba}\|}{\sqrt{m}} + \sqrt{n}\hat{\cR}_n(\tilde{\ba}).
$
Define $\cH_C :=\{\frac{1}{m}\sum_{j=1}^m a_j \phi(\bx;\bw^0_j): \frac{\|\ba\|}{\sqrt{m}}\leq C\}$. Then, we have $\rad_S(\cH_C)\lesssim \frac{C}{\sqrt{n}}$ \cite{shalev2014understanding}. Moreover, for any $\delta\in (0,1)$, with probability $1-\delta$ over the choice of training data, we have 
\begin{align*}
\cR(\hat{\ba}_n) &\lesssim \hat{\cR}_n(\hat{\ba}_n) + \rad_S(\cH_{C_{n,m}}) + \sqrt{\frac{\log(2/\delta)}{n}} \\
&\leq \hat{\cR}_n(\tilde{\ba}) + \frac{\|\tilde{\ba}\|}{\sqrt{nm}} + \frac{\|\tilde{\ba}\|}{\sqrt{nm}} + \hat{\cR}_n(\tilde{\ba}) + \sqrt{\frac{\log(2/\delta)}{n}}\\
&\leq 2 \hat{\cR}_n(\tilde{\ba}) + 2\frac{\|\tilde{\ba}\|}{\sqrt{nm}} + \sqrt{\frac{\log(2/\delta)}{n}}.
\end{align*}
By inserting Eqn. \eqref{eqn: emp-approx-rf}, we complete the proof. 
\qed

\section{Proofs for Section \ref{section barron approximation}} 

\begin{proof}[Proof of Theorem \ref{theorem direct linfty}]
Let $f(\bx) = \E_{(a,w)\sim\rho}\big[a\,\sigma(\bw^T\bx)\big]$. Using the homogeneity of the ReLU activation function, we may assume that $\|\bw\|_{\ell^1}\equiv 1$ and $|a| \equiv \|f\|_\B$ $\rho$-almost everywhere. By \cite[Lemma 26.2]{shalev2014understanding}, we can estimate
\begin{align*}
\E_{(a,\bw)\sim\rho^m}\left[\sup_{\bx\in[0,1]^d}\frac1m\sum_{i=1}^m\left(a_i\,\sigma(\bw_i^T\bx) -f(\bx)\right)\right]\leq 2\,\E_{(a,w)\sim\pi^m} \E_{\xi}\left[\sup_{\bx\in[0,1]^d}\frac1m\sum_{i=1}^m\xi_i\,a_i\,\sigma(\bw_i^T\bx) \right],
\end{align*}
where $\xi_i = \pm 1$ with probability $1/2$ independently of $\xi_j$ are Rademacher variables.
Now we bound 
\begin{align*}
\E_{\xi}\left[\sup_{\bx\in[0,1]^d}\frac1m\sum_{i=1}^m\xi_i\,a_i\,\sigma(\bw_i^T\bx) \right] &= \E_{\xi}\left[\sup_{\bx\in[0,1]^d}\frac1m\sum_{i=1}^m\xi_i\,|a_i|\,\sigma\big(\bw_i^T\bx\big) \right]\\
	&=  \E_{\xi}\left[\sup_{\bx\in[0,1]^d}\frac1m\sum_{i=1}^m\xi_i\,\sigma\big(|a_i|\,\bw_i^T\bx\big) \right]\\
	&\leq \E_{\xi}\left[\sup_{\bx\in[0,1]^d}\frac1m\sum_{i=1}^m\xi_i\,|a_i|\,\bw_i^T\bx \right]\\
	&=  \frac1m\E_{\xi}\left[\sup_{\bx\in[0,1]^d}\bx^T\sum_{i=1}^m\xi_i\,|a_i|\,\bw_i \right]\\
	&= \frac1m \E_{\xi}\left\|\sum_{i=1}^m\xi_i\,|a_i|\,\bw_i\right\|_{\ell^1}.
\end{align*}
In the first line, we used the symmetry of $\xi_i$ to eliminate the sign of $a_i$, which we then take into the (positively one-homogenous) ReLU activation function. The next inequality follows from the contraction lemma for Rademacher complexities, \cite[Lemma 26.9]{shalev2014understanding}, while the final equality follows immediately from the duality of $\ell^1$- and $\ell^\infty$-norm. Recalling that
\[
\|\by\|_{\ell^2} \leq \|\by\|_{\ell^1}\leq \sqrt{d+1}\,\|\by\|_{\ell^2}\quad \forall\ \by \in \R^{d+1}, \qquad \|a_i\,\bw_i\|_{\ell^1}= \|f\|_\B\quad\forall\ 1\leq i\leq m
\]
we bound
\begin{align*}
\E_{(a,\bw)\sim\rho^m}\left[\sup_{\bx\in[0,1]^d}\frac1m\sum_{i=1}^m\left(a_i\,\sigma(\bw_i^T\bx) -f(\bx)\right)\right] &\leq 2\sup_{\|\by_i\|_{\ell^1}\leq \|f\|_\B} \E_\xi\left\|\frac1m\sum_{i=1}^m \xi_i\by_i\right\|_{\ell^1}\\
	&\leq 2\,\|f\|_\B\,\sup_{\|\by_i\|_{\ell^1}\leq 1} \E_\xi\left\|\frac1m\sum_{i=1}^m\xi_i\by_i\right\|_{\ell^1}\\
	&\leq 2\sqrt{d+1}\,\|f\|_\B\sup_{\|\by_i\|_{\ell^2}\leq 1}\E_\xi\,\left\|\frac1m\sum_{i=1}^m\xi_i\by_i\right\|_{\ell^2}\\
	&\leq {2\,\|f\|_\B}{\sqrt{\frac{d+1}m}}
\end{align*}
by using the Rademacher complexity of the unit ball in Hilbert spaces \cite[Lemma 26.10]{shalev2014understanding}. Applying the same argument to $-\left[ \frac1m\sum_{i=1}^ma_i\sigma(\bw_i^T\bx)- f(\bx)\right]$, we find that
\[
\E_{(a,w)\sim\pi^m}\sup_{\bx\in[0,1]^d}\left|\frac1m\sum_{i=1}^m\left(a_i\,\sigma(\bw_i^T\bx) -f(\bx)\right)\right| \leq 4\,\|f\|_\B{\sqrt \frac{d+1}m}.
\]
In particular, there exists weights $(a_i, \bw_i)_{i=1}^m$ such that the inequality is true.

\end{proof}

\end{document}